\documentclass[journal]{IEEEtran}
\usepackage{cite}

\ifCLASSINFOpdf
   \usepackage[pdftex]{graphicx}
   \graphicspath{{../pdf/}{../jpeg/}}
   \DeclareGraphicsExtensions{.pdf,.jpeg,.png}
\else
\fi
\usepackage{amsmath,amsfonts,amssymb,amsthm}
\usepackage{xcolor}
\usepackage{subfigure}

\newtheorem{proposition}{Proposition}
\newtheorem{remark}{Remark}

\usepackage{tikz,xcolor,hyperref}
\definecolor{lime}{HTML}{A6CE39}
\DeclareRobustCommand{\orcidicon}{%
    \begin{tikzpicture}
    \draw[lime, fill=lime] (0,0) 
    circle [radius=0.16] 
    node[white] {{\fontfamily{qag}\selectfont \tiny ID}};    \draw[white, fill=white] (-0.0625,0.095) 
    circle [radius=0.007];    \end{tikzpicture}
    \hspace{-2mm}}
\foreach \x in {A, ..., Z}{%
    \expandafter\xdef\csname orcid\x\endcsname{\noexpand\href{https://orcid.org/\csname orcidauthor\x\endcsname}{\noexpand\orcidicon}}
    }

\begin{document}
%
\title{Robust Inertial-aided Underwater Localization based on Imaging Sonar Keyframes}
%
%
%

\author{Yang Xu, \IEEEmembership{Student Member, IEEE}, Ronghao Zheng, \IEEEmembership{Member, IEEE},  Senlin Zhang, \IEEEmembership{Member, IEEE}, \\
	and Meiqin Liu, \IEEEmembership{Senior Member, IEEE}


\thanks{Manuscript received November 8, 2021; revised January 17, 2022; accepted February 24, 2022. Date of publication February xx, 2021; date of current version March xx, 2021. The work of Xu and Zheng was financed by the NSFC-Zhejiang Joint Fund for the Integration of Industrialization and Informatization under Grant U1909206, Zhejiang Provincial Natural Science Foundation of China under Grant LZ19F030002, Natural Science Foundation of China under Grants 61873235 and 62173294, and a project supported by Scientific Research Fund of Zhejiang Provincial Education Department under Grant Y202148317. This work of Zhang and Liu was financed by the NSFC-Zhejiang Joint Fund for the Integration of Industrialization and Informatization under Grant  U1809212 and the Key Research and Development Program of Zhejiang Province under Grant 2019C03109. The Associate Editor coordinating the review process for this article was Dr. xxxx xxxx. (\textit{Corresponding author: Ronghao Zheng}.)}
	\thanks{Yang Xu, Ronghao Zheng and Senlin Zhang are with the College of Electrical Engineering, Zhejiang University, Hangzhou 310027, China,
		and also with the State Key Laboratory of
		Industrial Control Technology, Zhejiang University, Hangzhou 310027, China (e-mail: xuyang94@zju.edu.cn, rzheng@zju.edu.cn, slzhang@zju.edu.cn).}
	\thanks{Meiqin Liu is with the Institute of Artificial Intelligence and Robotics, Xi'an Jiaotong University, Xi'an 710049, China, and also with the State Key Laboratory of Industrial Control Technology, Zhejiang University, Hangzhou 310027, China (e-mail: liumeiqin@zju.edu.cn).}
	\thanks{Digital Object Identifier.}
}

%
%

\markboth{Journal of \LaTeX\ Class Files,~Vol.~14, No.~8, August~2015}%
{Xu \MakeLowercase{\textit{et al.}}: Robust Inertial-aided Underwater Localization based on Imaging Sonar Keyframes}
%



\maketitle

\begin{abstract}
This article focuses on the feature-based underwater localization and navigation for autonomous underwater vehicles (AUVs) using 2D imaging sonar measurements. The sparsity of underwater acoustic features and the loss of elevation angle in sonar images may introduce wrong feature matches or insufficient features for optimization-based underwater localization (i.e. under-constrained/degeneracy cases). This motivates us to propose a novel inertial-aided sliding window optimization framework to improve the estimation accuracy and the robustness to front-end outliers. Concretely, we first discriminate under-constrained/well-constrained sonar frames and define sonar keyframes (SKFs) based on the Jacobian matrix derived from odometry and sonar measurements. To utilize the past well-constrained SKFs mostly, we design a size-adjustable windowed back-end optimization scheme based on singular values. We also prove that the landmark triangulation failure (navigation problem) caused by sonar motion can be solved in 2D scenes. Comparative simulation and evaluation on a public dataset show the proposed method outperforms the existing ones in pose estimation and robustness even without loop closure, and also ensure the real-time performance for online applications.
\end{abstract}


%
\IEEEpeerreviewmaketitle

\section{Introduction}
\label{sec:introduction}
\IEEEPARstart{U}{nderwater} autonomous localization and navigation with no prior map become more essential for autonomous underwater vehicles (AUVs) to complete various kinds of missions. However, the electromagnetic signals from the global positioning system attenuate quickly underwater, and the external underwater acoustic positioning systems (e.g. ultra-short/short/long baseline) are restricted by the limited and expensive beacons~\cite{paull2013auv,wu2019survey,zhang2021efficient}, thus the internal aided navigation (IAN) systems are more preferred. 

Among the IAN systems, the widely used inertial navigation system (INS) accumulates error rapidly over time, and the visually augmented navigation (VAN) systems are always subject to close sensing range, visibility, and illumination conditions {since many feature matching failures may occur in complex underwater environments such as turbid water and visual feature-less regions~\cite{eustice2008visually,yang2020extrinsic}. However, acoustic sensors such as side-scan sonars, multibeam sonars, and imaging sonars are more appropriate in these scenes,} among which the imaging sonars perform better in underwater object detection, collision avoidance, and acoustic visual navigation \cite{teng2020robust}.

Several practices only use imaging sonar image sequences to estimate sonar poses, such as 2D optical flow-based method \cite{henson2018attitude} and 3D motion estimation \cite{sekkati20073,negahdaripour20133}. 
Assuming a planar seafloor, Negahdaripour~\cite{negahdaripour2008bundle} utilized multiple view bundle adjustment (BA) to estimate the 3D sonar motion just from sonar image flow containing 3D objects and their shadows on the seafloor. However, underwater pose estimation only using imaging sonars may get less reliable results than the one using additional multi-sensor measurements because of the feature sparsity and low resolution in noisy sonar images.

To better locate and navigate the AUVs by combining imaging sonar and other navigation sensors, {extended Kalman filter (EKF)-based} estimation methods have already been applied to fuse multi-sensor measurements in semi-structured or structured underwater environments such as harbors, marinas, and ship hull inspection~\cite{ribas2008underwater, mallios2014scan, walter2008slam}. Mallios \textit{et al.} utilized an EKF to estimate the local pose increment to correct the acoustic image distortions produced by vehicle motion, and an augmented state EKF to estimate and keep {the registered scans' poses} \cite{mallios2010probabilistic}. 
Chen \textit{et al.} \cite{chen2019rbpf} proposed a {Rao-Blackwellized particle filter-based simultaneous localization and mapping (SLAM)} algorithm for an AUV equipped with a slow mechanically scanning imaging sonar. 
Lin \textit{et al.} \cite{lin2020gated} presented a gated recurrent unit-based particle filter using sonar measurements to improve the state estimation of AUVs.

Nevertheless, the main defects of EKF-based underwater SLAM approaches involve the increasing scale of state vector and the covariance matrix and a resulting higher computational complexity, especially in a large underwater scene \cite{li2018robust, sekkati20073}. 

Instead, acoustic BA methods{\cite{negahdaripour2008bundle,ShinBundle}} can fuse multi-source measurements in an optimization way, which outperforms in precision and robustness than EKF-based methods when processing potential loop closure and front-end failure.

\section{Related Works}
\label{relate}
In an earlier time, Shin \textit{et al.}~\cite{ShinBundle} proposed a two-view acoustic bundle adjustment (ABA) framework for an AUV mapping the seafloor. 
Notably, the special acoustic imaging principle and the missing vertical bearing angle are found to bring about the spurious motion, i. e. the inherent ambiguities of 3D motion and scene structure interpretation, as Negahdaripour explored and analyzed in \cite{negahdaripour2012visual} from 2D forward imaging sonar image sequences using multi-view geometry. More recently, Huang \textit{et al.} \cite{huang2015towards, huang2016incremental} observed several landmark degeneracy cases when sonar moves in special directions.  
In other words, the sparsity of underwater acoustic features \cite{westman2018feature} and the loss of elevation angle in 2D sonar images \cite{pyo2019site}, degeneracy cases occur in SLAM back-end when it comes to wrong data association and insufficient sonar measurements, namely under-constrained cases in the back-end optimization~\cite{westman2019degeneracy,huang2015towards}, such as relative pose ambiguity and landmark cannot be triangulated. This is also called state unobservable in EKF-based SLAM~\cite{yang2017acoustic,yang2019observability}.

To solve this, Huang and Kaess \cite{huang2015towards, huang2016incremental} presented the acoustic structure from motion (ASFM) method to optimize both sonar poses and 3D feature positions in this circumstance. They also used onboard navigation sensors such as IMU and {Doppler Velocity Log (DVL)} to mitigate the degeneracy phenomenon.
Additionally, with the learning-based loop closure mechanism, Li \textit{et al.} \cite{li2018pose} proposed a pose-graph imaging sonar SLAM scheme based on ASFM to handle the degeneracy cases and improve accuracy in ship hull inspection.
However, successful loop closure mainly depends on rich acoustic features, while may fail in the unstructured scene and bring even larger errors. The useful constraints in previous sonar frames have also been neglected.

So far, how to extract and utilize constraints from past sonar images to enhance the accuracy and robustness of the two-view ABA to outliers (even without loop closure), then achieve the long-term underwater localization and navigation is still an open problem.

Motivated by these related works~\cite{ShinBundle,westman2018feature}, in this paper, we propose a new approach to enhance the accuracy and robustness of the two-view ABA using imaging sonars. The main contributions can be summarized as follows:
{\begin{enumerate}
    \item As an extension of previous work \cite{xu2020keyframe}, using the information matrix derived from odometry and sonar measurements, we further analyze and prove the landmark triangulation failure caused by special sonar motions can be eliminated in 2D scenes (navigation problem). 
    \item On this theoretical basis, we propose an approach to discriminate the under-constrained/well-constrained sonar frames using the singular values of the Jacobian matrix, and define the sonar keyframes (SKFs) to be chosen.
    \item Moreover, we present a novel inertial-aided sliding window optimization framework to improve the pose estimation accuracy and the robustness to front-end outliers and make full use of the past well-constrained SKFs when confronting degeneracy cases. The size-adjustable window optimization also ensures time efficiency in field applications.
\end{enumerate}}

Additionally, simulations carried by a simulated remotely operated vehicle (ROV) equipped with an imaging sonar show significant improvement in the accuracy of pose and landmark estimation, especially the robustness to outliers from front-end, when confronting feature sparsity and wrong feature matches. The evaluation of the marina dataset also validates our method outperforms two-view ABA and INS DR by decreasing 1.38 m and 5.95 m of average trajectory root mean square errors (RMSEs) in a $110~m\times200~m$ trajectory estimation.

The rest of this paper is organized as follows. We present the inertial-aided 2D imaging sonar localization and navigation in Section~\ref{ins} and analyze the degeneracy cases in Section~\ref{sec3}. In Section~\ref{sec4}, we define the SKFs and propose a SKF-based elastic windowed optimization scheme. Comparative experiment results and discussions are given in Section~\ref{sec5}. Concluded remarks and future work are provided in Section~\ref{sec6}.

\section{Inertial-aided Imaging Sonar Localization and Navigation}
\label{ins}
In this section, we consider a feature-based localization and navigation problem using imaging sonar and low-cost inertial navigation sensors.
We briefly introduce the imaging sonar model and the inertial model within the factor-graph optimization framework as follows, which will be the basis for our proposed method.
\subsection{Imaging Sonar Model}
A typical 2D forward-looking imaging sonar model is shown in Fig.~\ref{sonar}. 
\label{sec21}
\begin{figure}
	\centering
	\includegraphics[width=0.35\textwidth]{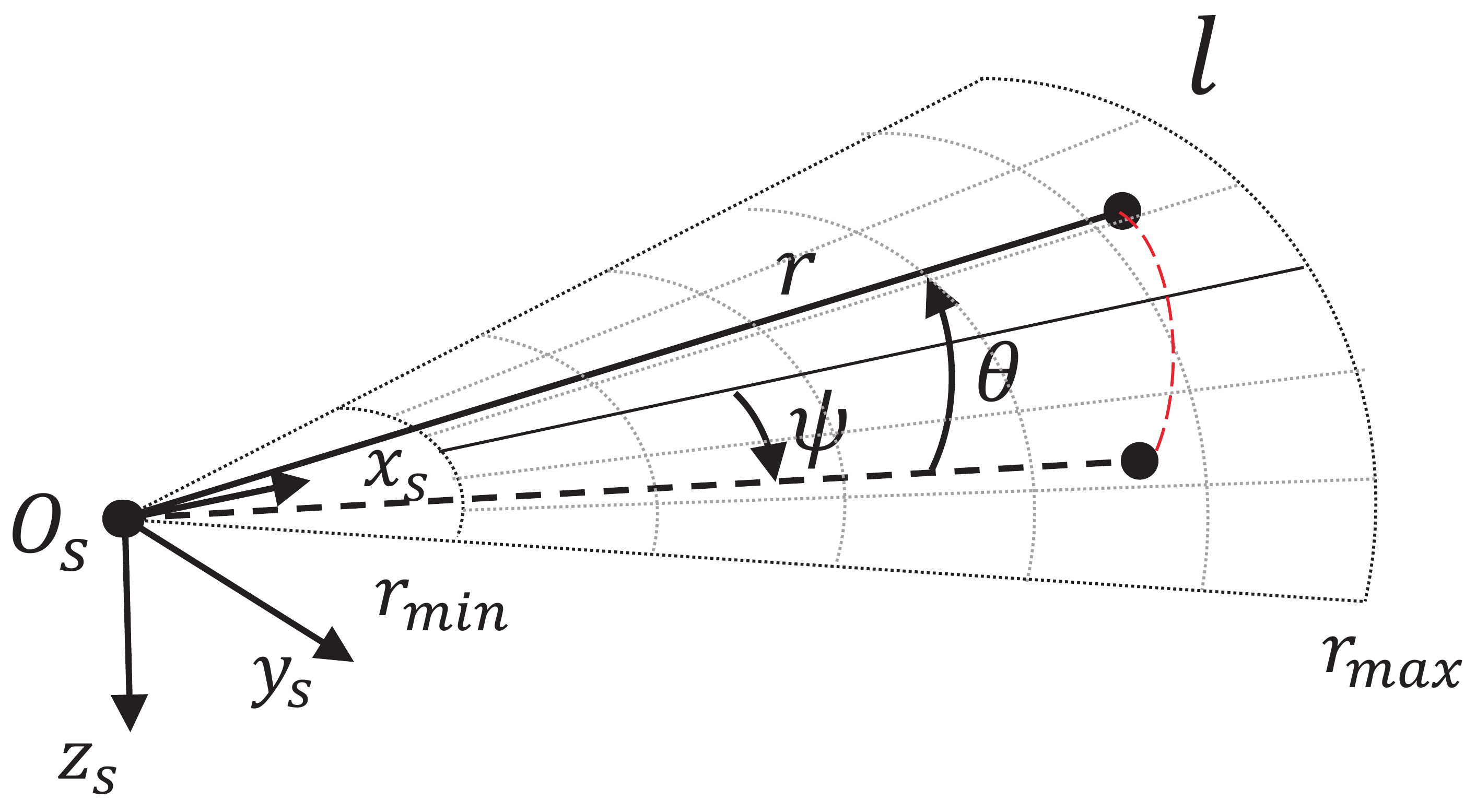}
	\caption{A typical imaging sonar model. All points on the red arc will be projected into a single pixel point on the 2D sonar image and the elevation angle $\theta$ is lost in the sonar image.}
	\label{sonar}       
\end{figure}
Here we use the sonar spherical coordinates to parameterize a feature point $l =[\psi,r,\theta]^\mathrm T$, thus the coordinates transformation can be described as follows:
\begin{align}\label{eq:model}
	&\pi(\mathbf p^l)=\left[\begin{matrix}\psi\\r\end{matrix}\right]=\left[\begin{matrix} \mathrm{atan2}(p_{y}^l,p_{x}^l)\\ \|\mathbf p^l\|_2^2 \end{matrix}\right],\\ \nonumber
	&\mathbf p^l=\left[\begin{matrix}p_{x}^l,~p_{y}^l,~p_{z}^l\end{matrix}\right]^{\mathrm{T}}=\left[\begin{matrix} r \cos \psi \cos \theta\\r \sin \psi \cos \theta\\ r \sin \theta \end{matrix}\right],
\end{align}
where $\mathbf p^l$ is the sonar Cartesian coordinates of point $l$, $r\in[r_{min},r_{max}]$ is the measured range, and $\psi$ is the horizontal bearing angle. $\pi(\cdot)$ is the projection function which projects $l$ to a single pixel point at the 2D sonar image.
Note that the sonar images lost the elevation angle $\theta$ of acoustic features in the 3D-2D projection. The sonar vertical aperture and the horizontal field of view (FOV) meets $\theta\in[\theta_{min},\theta_{max}]$ and $\psi\in[\psi_{min},\psi_{max}]$, respectively.

\subsection{Underwater Inertial Navigation}
\label{sec:ian}

Consider an IMU kinematic model as follows \cite{miller2010autonomous,forster2016manifold}:
\begin{align}
    &\dot{\mathbf{p}}_W(t)=\mathbf{v}_W(t),\\
    &\dot{\mathbf{v}}_W(t)=\mathbf{R}_{WB}^\mathrm{T}(t)(\mathbf{\tilde a}^m_{B}(t)-\mathbf{b}_a(t)-\boldsymbol{\eta}_a(t))+\mathbf g_W,\\
	&\dot{\mathbf{R}}_{WB}(t)=\mathbf{R}_{WB}(t) \Omega \left(\mathbf{\tilde \omega}_{WB}^m(t)-\mathbf{b}_g(t)-\boldsymbol{\eta}_g(t)\right),\\
	&\dot{\mathbf{b}}_g(t)=\boldsymbol{\eta}_g(t),~\dot{\mathbf{b}}_a(t)=\boldsymbol{\eta}_a(t),
\end{align}
where $\mathbf{p}_W$ and $ \mathbf{v}_W$ are the position and velocity in the world frame $W$, respectively. $\mathbf{R}_{WB}$ is the rotation matrix from world frame to body frame $B$, $\mathbf{b}_g$ and $\mathbf{b}_a$ are the biases of gyroscope and accelerometer from IMU. 
$\Omega(\omega)$ is the skew-symmetric operator for a vector $\omega\in \mathbb{R}^{3\times1}$. $\mathbf{\tilde \omega}_{WB}^m$ and $\mathbf{\tilde a}_{B}^m$ are the angular velocity and linear acceleration measured directly by IMU. $\mathbf g_W$ is the gravity acceleration constant vector in the world frame. $\boldsymbol{\eta}_g$ and $\boldsymbol{\eta}_a$ are the Gaussian noises in the IMU measurements.

Fig.~\ref{time} shows the differences of update frequency, where IMU updates faster than the imaging sonar frames, and the sonar keyframes are chosen every several frames. 
It is worth noting that if not using the relative pose increment between two consecutive sonar frames but the absolute pose estimation at every frame, the pose error caused by dead reckoning will accumulates fast and get significant.
Therefore, we only concern about the pose increment $\Delta \mathbf{p}_W$  and $\Delta \mathbf{R}_{WB}$ in a time interval $[t,t+\Delta t]$:
\begin{align}
	&\Delta \mathbf{v}_W = \mathbf{R}_{WB}^{\mathbf T}(t)\mathbf{s}(t)+{\mathbf{g}}_W\Delta t,\\
	&\Delta \mathbf{p}_W =  \mathbf{v}_W(t)\Delta t+\mathbf{R}_{WB}^{\mathbf T}(t)\mathbf{y}(t)+\frac{1}{2}\mathbf{g}_W\Delta t^2,\\
	&\mathbf{R}_{WB}(t+\Delta t) = \mathbf{R}_{WB}(t)\Omega(\mathbf{u}(t)),
\end{align}
where $\mathbf{s},~\mathbf{y}$ and $\mathbf{u}$ are integral terms \cite{li2013high}.

\subsection{Feature-based Two-view Acoustic BA}
\label{sec22}
Assuming Gaussian noises in the odometry and sonar measurements, we can get the following measurement functions:
\begin{align}
	\mathbf z_{odom}^i&=\boldsymbol f(x_{i-1},x_i)+\mathcal{N}(0,\mathbf \Lambda_i), \\
	\mathbf z_{sonar}^k&={\boldsymbol h}(x_{ik},l_{jk})+\mathcal{N}(0,\mathbf \Sigma_k),
\end{align}
where $x_i=[\psi_{x_i},\varphi_{x_i},\phi_{x_i},t_{x_i}^x,t_{x_i}^y,t_{x_i}^z]^\mathrm T(i=1,2,3,\dots)$ represent sonar pose and $l_j(j=1,\dots,M)$ {denotes the feature coordinate defined on previous pose}. $\mathbf z_{sonar}^k(k=1,\dots,N)$ is the $k$th sonar measurement associated with pairwise pose $x_{ik}$ and feature $l_{jk}$. For notion brevity, here we use $i$ to denote the current frame at $t+\Delta t$ and $i-1$ for frame at $t$.

\begin{figure}
	\centering
	\includegraphics[width=0.4\textwidth]{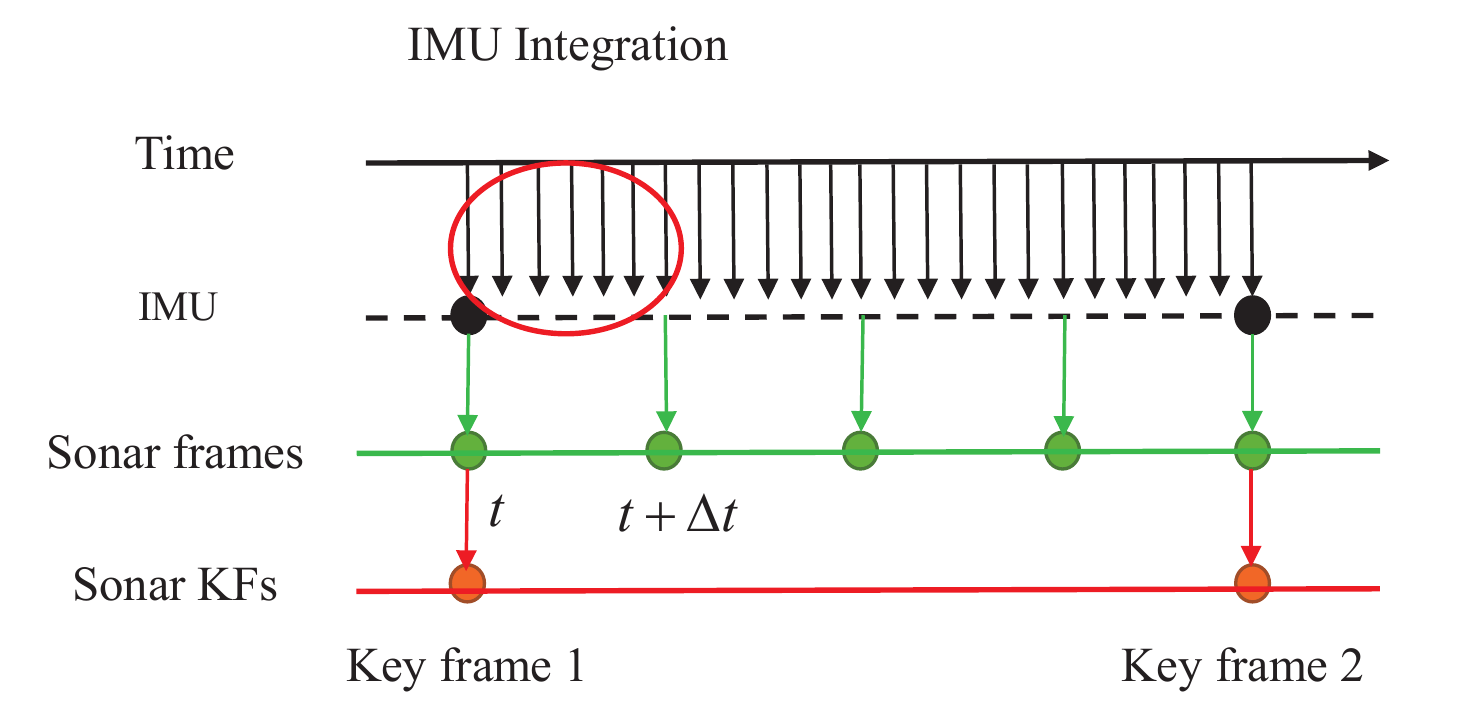}
	\caption{An illustration of different sensor update rates.
		The numerically pre-integrated pose increment in the short time interval $[t, t+\Delta t]$ of two consecutive sonar frames is more accurate and reliable instead of the continuous increment of odometry measurement.}
	\label{time}       
\end{figure}
The odometry prediction function represented by $\boldsymbol f(\cdot)$ is acquired from the inertial navigation in Sect.~\ref{sec:ian}.
The sonar prediction function $\boldsymbol h(x_{ik},l_{jk})$ transforms the features $l_{jk}$ associated with pose $x_{ik}$ onto sonar polar coordinates by using $\pi(\cdot)$ defined in Eq.~\eqref{eq:model}, then gets the predicted 2D polar coordinates $\hat{\mathbf z}_{sonar}^k$.
\begin{figure*}
	\centering
	\includegraphics[width=0.8\textwidth]{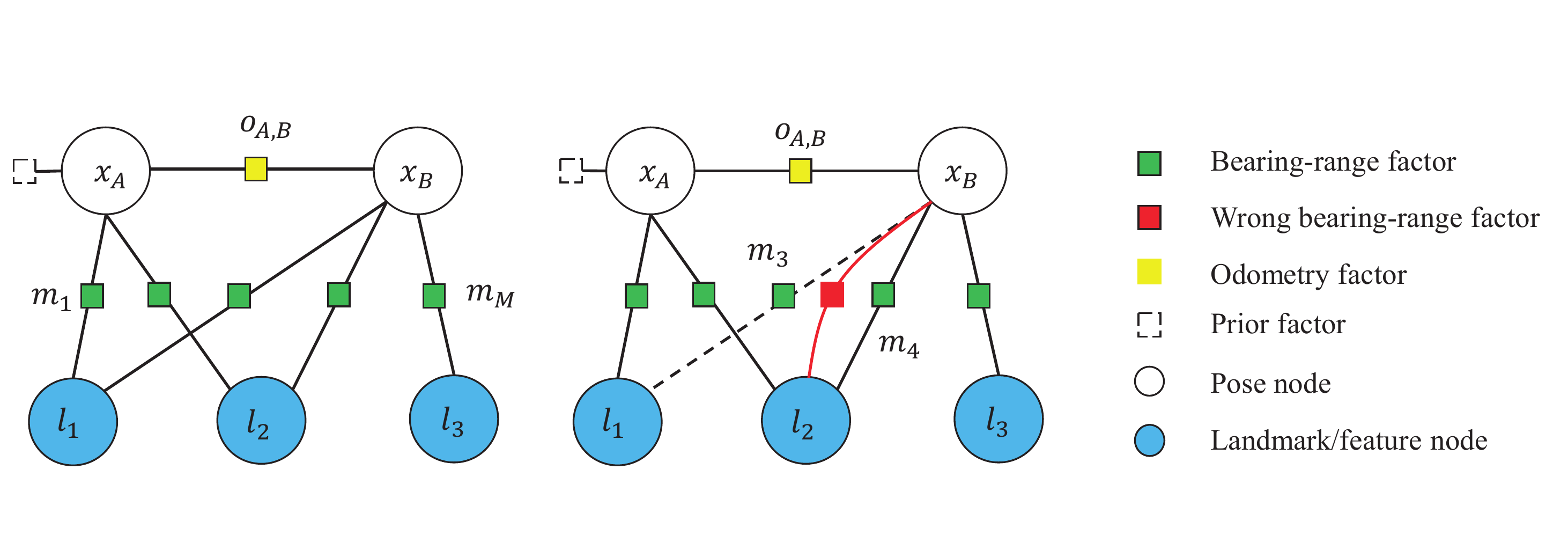}
	\caption{Illustration of factor-graph based two-view ABA. \emph{Left}: standard two-view ABA, \emph{Right}: two-view ABA with wrong association. Sonar measurement factors (green squares) connect one pose (white circle) and one landmark (blue circle). Odometry factor (yellow square) connects two consecutive poses. The prior factor is depicted as a dashed square. The red square factor and dashed line represent a wrong association.}
	\label{factorgraph}       
\end{figure*}
%

It is worth noting that in two-view ABA, the previous pose $x_{i-1}$ denoted by $x_A$ is set to be constant zero as a reference, and the current pose $x_i$ denoted by $x_B$ represents the relative pose between these two consecutive sonar frames. The prior factor is omitted here.
Therefore, for each feature point $l_j$ measured at $x_A$, the predicted sonar measurement at pose $x_B$ is as follows:
\begin{align}
	& \hat{\mathbf z}_{sonar}^k=\boldsymbol h(x_{ik},l_{jk})=\pi(\mathbf q_{jk}),\\ \nonumber
	& \left[\begin{matrix} \mathbf q_{jk}\\1 \end{matrix}\right]= \mathbf T_{x_B}^{-1} \left[\begin{matrix} \mathbf p_{jk}^l\\1 \end{matrix}\right], \mathbf T_{x_B}=\left[\begin{matrix} \mathbf R_{x_B}&\mathbf t_{x_B}\\ \mathbf 0^T&1 \end{matrix}\right],
\end{align}
where $\mathbf p_{jk}^l$ is the feature's Cartesian coordinates at sonar pose $x_A$ and $\mathbf q_{jk}$ is the corresponding reprojected coordinates at pose $x_B$, $\mathbf T_{x_B}\in \mathbb{R}^{4\times 4}$ is the homogeneous transformation matrix, $\mathbf R_{x_B}=\mathfrak{Rot}(\psi_{x_B},\varphi_{x_B},\phi_{x_B})$ and $\mathbf t_{x_B}=[t_{x_B}^x,t_{x_B}^y,t_{x_B}^z]^T$ is the relative 3D rotation matrix and translational vector, respectively.

Therefore, the ABA problem converts into a Maximum a Posteriori (MAP) problem: Given all the sonar and odometry observations $\mathbf z=\{\mathbf z_{sonar}^1,\dots,\mathbf z_{sonar}^k,\dots,\mathbf z_{sonar}^N,\mathbf z_{odom}\}$ at $x_A$ and $x_B$, respectively, with no proir knowledge, find a maximum posterior probability set of poses and landmarks $\mathbf \Theta=\{x_B,l_1,\dots,l_M\}$ according to the Bayesian Rule:
\begin{align}\label{eq:MAP}
	\Theta^*&=\mathop{\arg\max}_{\Theta} p(\mathbf \Theta|\mathbf z) \propto \mathop{\arg\max}_{\Theta} p(\mathbf z | \mathbf \Theta) \\ \nonumber
	&=\mathop{\arg\max}_{\Theta} \prod_{k=1}^N p(
	\mathbf z_{sonar}^k, \mathbf{z}_{odom}^i
	|\mathbf \Theta 
	).
\end{align}

As the left sub-figure in Fig.~\ref{factorgraph}, we introduce the factor graph optimization~\cite{FactorGraph} to model the above MAP problem. Since the conditional probability density function of each measurement can be expressed as  negative exponential form:
\begin{align}
	p(\mathbf z_{sonar}^k|x_{ik},l_{jk}) \propto \exp
	\left\{-\frac{1}{2}\|\boldsymbol h-\mathbf z_{sonar}^k\|_{\mathbf \Sigma_k}^{2}
	\right\},\\ \nonumber
	p(\mathbf z_{odom}^i|x_{i-1},x_{i}) \propto \exp
	\left\{-\frac{1}{2}\|\boldsymbol f-\mathbf z_{odom}^i\|_{\mathbf \Lambda_i}^{2}
	\right\},
\end{align}
then we can get a nonlinear least-squares (NLS) problem by taking the natural logarithm of Eq.~\eqref{eq:MAP}:
\begin{align}\label{eq:nls}
	\Theta^{*}&=\underset{\Theta}{\operatorname{argmin}}\sum_{k=1}^{N}\left\| \boldsymbol h\left(x_{ik}, l_{jk}\right)-\mathbf z_{sonar}^k\right\|_{\mathbf \Sigma_{k}}^{2}\\ \nonumber
	&+\sum_{i=1}^{N}\left\| \boldsymbol f\left(x_{i-1}, x_{i}\right)-\mathbf z_{odom}^i\right\|_{\mathbf \Lambda_{i}}^{2}.
\end{align}
where $||\cdot||^2_{\Sigma}$ denotes the Mahalanobis norm. Note that the odometry term can be incorporated in the relative pose $x_B$ to be estimated.

Applying first-order Taylor expansion for Eq.~\eqref{eq:nls} at the linearized point $ \Theta^0=\{x_{i}^0,l_{j}^0\}$ and transforming the Mahalanobis norm to Euclidean norm, we can get a linear least-squares (LLS) problem as follows: 
\begin{align}\label{eq:lls}
	\mathbf \Delta^*& \approx \mathop{\arg \min}_\Delta \sum_{k=1}^N \left\|\boldsymbol h( \Theta^0)+ \mathbf H_k \mathbf \Delta-\mathbf z_{sonar}^k\right\|_{\mathbf \Sigma_{k}}^{2}\\ \nonumber
	&=\mathop{\arg \min}_\Delta \sum_{k=1}^N \left\|\mathbf A_k \mathbf \Delta-\mathbf b_k \right\|^2
	=\mathop{\arg \min}_\Delta \left\|\mathbf A \mathbf \Delta-\mathbf b \right\|^2, \nonumber
\end{align}
where 
$\mathbf A_k=\mathbf \Sigma_k^{-1/2}\mathbf H_k ,~\mathbf b_k=\mathbf \Sigma_k^{-1/2}(\mathbf z_{sonar}^k-\mathbf h( \Theta^0))$
are the \textit{whitened} Jacobian matrix and the error vector, respectively. $\mathbf \Sigma_k^{-1/2}$ is the square-root information matrix and
 $\mathbf \Delta^*=\Theta- \Theta^0$ 
is the update vector to be solved.

The closed-form $\mathbf{A}$ is more appropriate in back-end optimization for its computational efficiency, rather than numerical computation iteratively, i.e. ($k$ is omitted here):
\begin{align}\label{eq:jacobian}
	\frac{\partial \boldsymbol h(x_{i},l_{j})}{\partial x_{i}}&=\frac{\partial \boldsymbol h(x_{i},l_{j})}{\partial \mathbf q_j} \frac{\partial \mathbf q_j}{\partial x_i},\\ \nonumber
	\frac{\partial \boldsymbol h(x_{i},l_{j})}{\partial l_{j}}&=\frac{\partial \boldsymbol h(x_{i},l_{j})}{\partial \mathbf q_j} \frac{\partial \mathbf q_j}{\partial \mathbf p_j^l} \frac{\partial \mathbf p_j^l}{\partial l_j},
\end{align}
where, specifically, the sub-matrices composing a whole Jacobian matrix are as follows:
\begin{align}
	&\frac{\partial \boldsymbol h(x_A,l_j)}{\partial x_B}=\mathbf{0}_{2\times 6},~
	\frac{\partial \boldsymbol h(x_A,l_j)}{\partial l_j}=\left[\begin{matrix} 1 &0 &0\\0 &1 &0 \end{matrix}\right],\\ \nonumber
	&\frac{\partial \boldsymbol h(x_B,l_j)}{\partial x_B}=
	\left[\frac{\partial \boldsymbol h(x_B,l_j)}{\partial \mathbf q_j}\right] \begin{bmatrix} \Omega(\mathbf q_j) & -\mathbf I_{3\times3} \end{bmatrix},\\ \nonumber
	&\frac{\partial \boldsymbol h(x_B,l_j)}{\partial l_j}= 
	\left[\frac{\partial \boldsymbol h(x_B,l_j)}{\partial \mathbf q_j}\right]
	\left[\mathbf R_{x_B}^T \right]
	\left[\frac{\partial \mathbf p_j^l}{\partial l_j}\right],\\ \nonumber
	&\frac{\partial \boldsymbol h(x_B,l_j)}{\partial \mathbf q_j}=\begin{bmatrix} -\frac{q_{jy}}{q_{jx}^2+q_{jy}^2} &\frac{q_{jx}}{q_{jx}^2+q_{jy}^2} &0\\\frac{q_{jx}}{\|\mathbf q_j\|_2} &\frac{q_{jy}}{\|\mathbf q_j\|_2} &\frac{q_{jz}}{\|\mathbf q_j\|_2} \end{bmatrix},\\ \nonumber
	&\frac{\partial \mathbf p_j^l}{\partial l_j}=\begin{bmatrix} -r_j\cos \theta_j \sin \psi_j &\cos \psi_j \cos \theta_j &-r_j\cos \psi_j \sin \theta_j\\r_j\cos \theta_j \cos \psi_j &\sin \psi_j \cos \theta_j &-r_j \sin \psi_j \sin \theta_j \\ 0&\sin \theta_j &r_j \cos\theta_j \end{bmatrix}, \nonumber
\end{align}
where $\Omega(\mathbf q_j)$ is a skew-symmetric matrix as:
\begin{align}
	\Omega(\mathbf q_j)=
	\left[
	\begin{matrix}
		0 & -q_{jz} & q_{jy} \\
		q_{jz} & 0 & -q_{jx} \\
		-q_{jy}& q_{jx} & 0 
	\end{matrix}
	\right],
	\mathbf q_j=
	\left[
	\begin{matrix}
		q_{jx} \\
		q_{jy} \\
		q_{jz} 
	\end{matrix}
	\right].
\end{align}

Then we can get the closed-form $\mathbf A$ by splicing these sub-matrices along the column representing poses and landmarks and the row representing sonar measurements. A typical example is shown in Fig.~\ref{jacob}(a), the factor graph is corresponded to the left figure of Fig.~\ref{factorgraph}.

\section{Degeneracy Cases Analysis In Two-view ABA}
\label{sec3}
\subsection{Relative Pose Ambiguity}
\label{sec31}
We first consider the degeneracy case of relative pose ambiguity between two viewpoints with 6 DOFs. Especially, the sparse insufficient observations and wrong feature matches may directly lead to this. We present the necessary condition for the number of needed landmark observations between two sonar frames from the existence on the solutions of LLS problem mentioned above.

Generally, in the two-view ABA optimization, we assume all $M$ landmarks can be observed from two poses so that the maximum number of all known observation equations is $4M$. {There are $6+3M$ unknown variables containing 6 from pose $x_B=[\psi_{x_B},\varphi_{x_B},\phi_{x_B},t_{x_B}^x,t_{x_B}^y,t_{x_B}^z]^\mathrm T$ and $3M$ from the position $\mathbf p^l = \left[\begin{matrix}p_{x}^l,~p_{y}^l,~p_{z}^l\end{matrix}\right]^{\mathrm{T}}$ of $M$ landmarks.} The necessary condition to ensure the solvability of the LLS is:
\begin{equation}
	4M\geq 6+3M,
\end{equation}
i.e., $M \geq 6$, which means that if the number of features is less than 6, the back-end LLS problem, i.e. Eq.~\eqref{eq:lls}, cannot be fully constrained. This exactly implies that too sparse features may lead to underdetermined LLS and pose degeneracy. Fortunately, we can introduce more absolute sensors to reduce the number of unknown variables so that the required features can be fewer.

Besides, an alternative approach referred to the numerical optimization theory can be used to solve the underdetermined LLS problem, requiring the normal equation: 
\begin{align}\label{eq:normal}
	\mathbf A^{\rm T} \mathbf A \mathbf \Delta^* = \mathbf A^{\rm T} \mathbf b,
\end{align}
where $\mathbf A\in\mathbb R^{p\times q}(p,~q \in \mathbb{N}^+)$ is a full column rank Jacobian matrix, i.e. $p\geq q$. $\mathbf A^{\rm T} \mathbf A$ is the positive definite Hessian matrix. A conventional solution to this equation is to introduce the pseudo inverse $(\mathbf A^{\rm T} \mathbf A)^{-1}\mathbf A^{\rm T}\mathbf b$. In addition, the QR and Cholesky decomposition methods are more preferred in the sparse optimization problem such as SLAM \cite{nocedal2006numerical}. 

However, the above-mentioned methods cannot handle the degeneracy caused by the inevitable front-end perturbations or even wrong feature matches even using modern front-end algorithms. 
Here we raise an intuitive example. Consider a wrong association between two consecutive sonar frames. As the right figure in {Fig.~\ref{factorgraph}}, $x_i(i=A,~B)$ and $l_j(j=1,2,3)$ are neighboring poses and landmarks marked with white and blue circles respectively, and $m_k(k=1,\dots,M, M=5)$ are sonar measurements marked with green squares. The dashed line between $l_1$ and $x_B$ depicts that $m_3$ is expected to associate with them, but actually, it associates with $l_2$ and $x_B$ due to the wrong feature match. The red square marker represents the actual wrong association. 

More specifically, the resulting effects on the back-end optimizer of wrong association are evidently embodied in the underlying Jacobian matrix.  {As in Fig.~\ref{jacob}, the green rectangles represent the Jacobian sub-matrices with respect to the measurement factor $m_k$ and the landmark $l_j$, and orange rectangles stand for the Jacobian sub-matrices related to the measurement factor $m_k$ and pose $x_i$, and the white ones are zero sub-matrices for irrelevant measurement factors and variables.} The immediate adverse effect is that the rows of the whole Jacobian matrix related to $m_3$ and $m_4$ become all same, i.e., the actual Jacobian matrix structure in Fig. \ref{jacob}(a) becomes rank-deficient in Fig. \ref{jacob}(b) at the optimizer. Consequently, the Hessian matrix also gets no longer positive definite, leading to unstable numerical computation and larger error.
\begin{figure}
	\centering
	\subfigure[]{
		\begin{minipage}[t]{0.45\linewidth}
			\centering
			\includegraphics[width=1.55in]{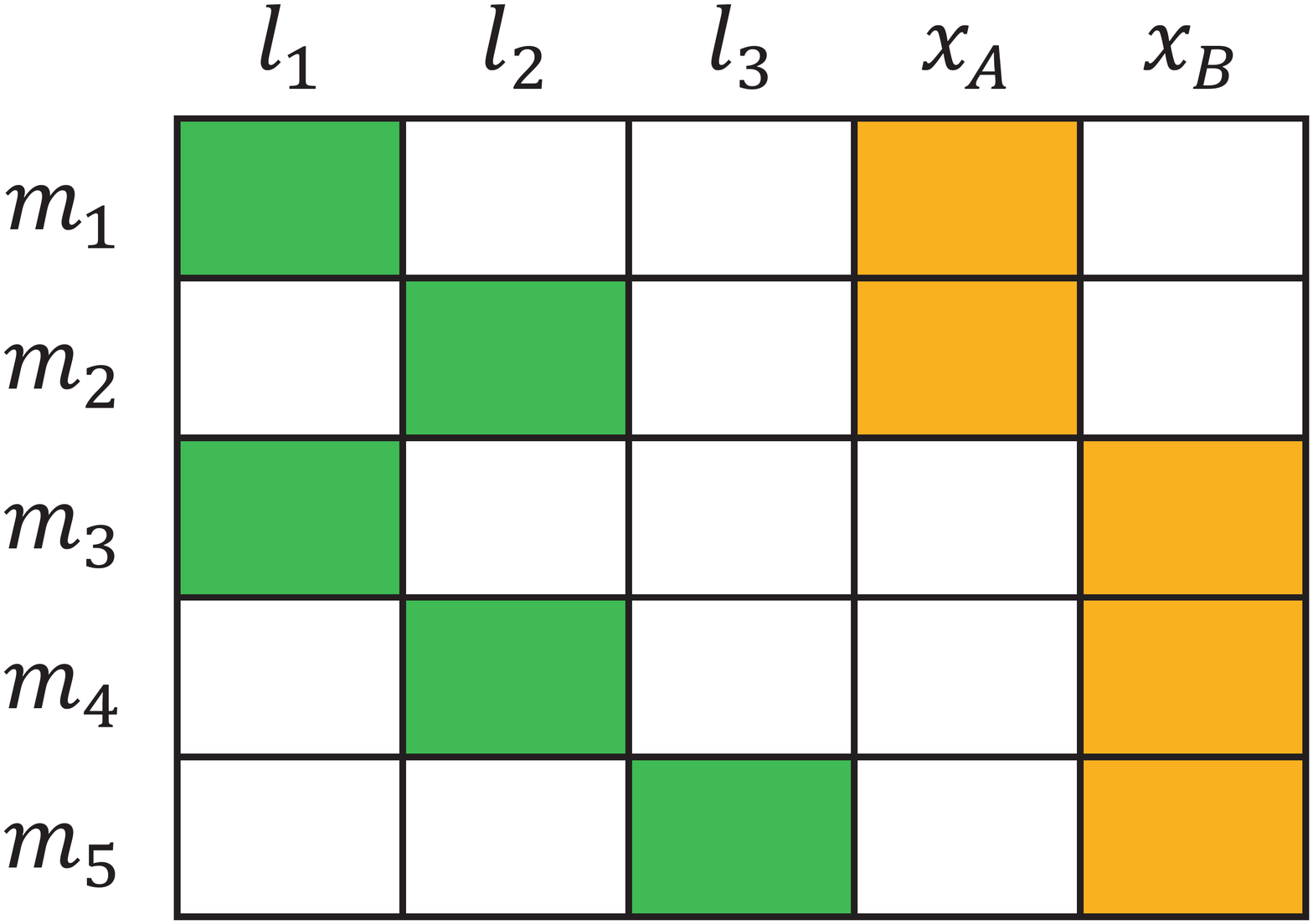}
		\end{minipage}%
	}%
%
	\subfigure[]{
		\begin{minipage}[t]{0.5\linewidth}
			\centering
			\includegraphics[width=1.55in]{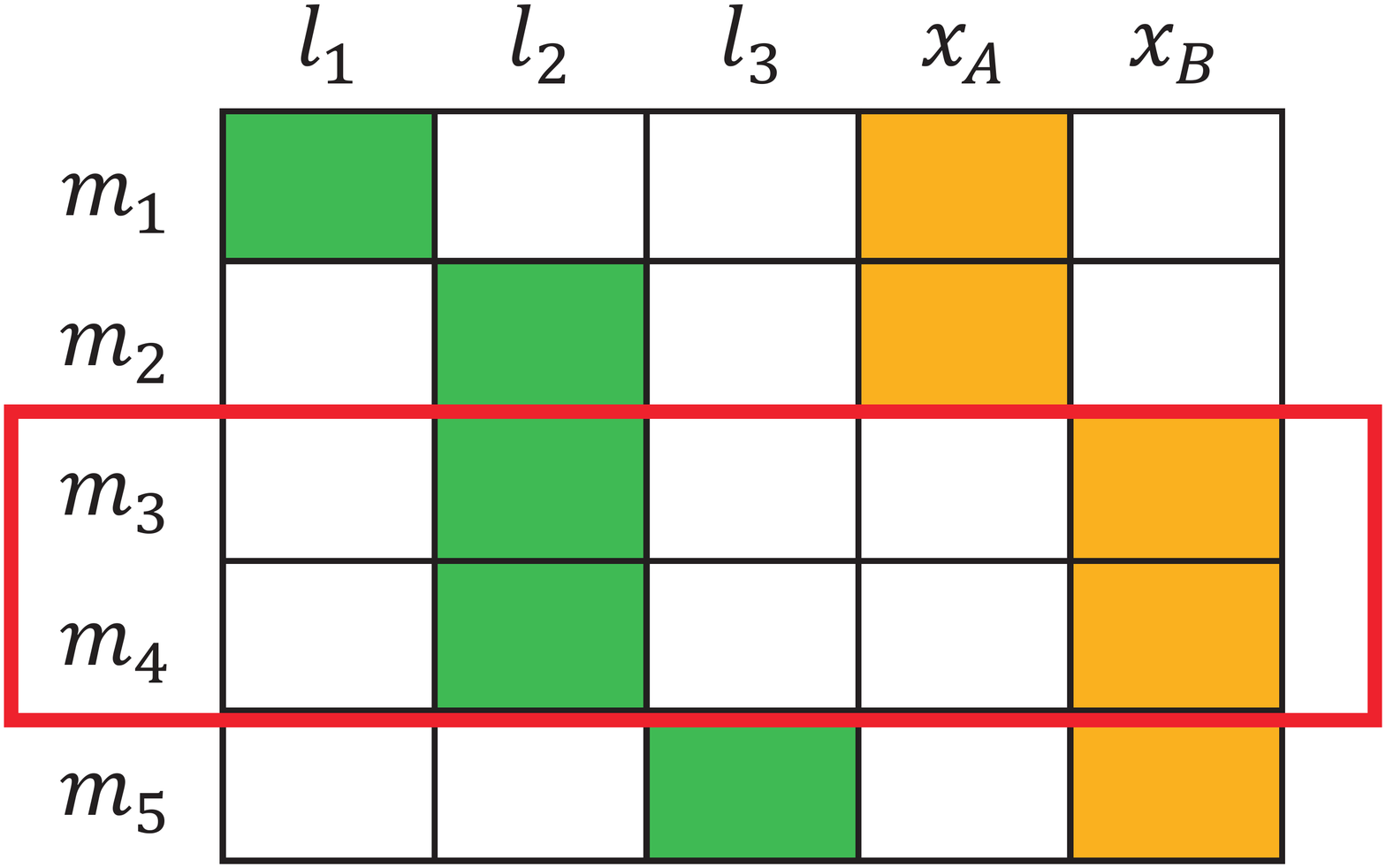}
		\end{minipage}
	}%
	\centering
	\caption{Illustration of Jacobian matrix in a wrong association example. (a) Expected Jacobian structure. (b) Actual Jacobian structure in the optimizer. The rows of actual Jacobian matrix related to $m_3$ and $m_4$ become linear correlated approximately here, leading the Hessian matrix to be non-positive definite.}
	\label{jacob}
\end{figure}
\begin{figure}
	\centering
	\includegraphics[width=0.2\textwidth]{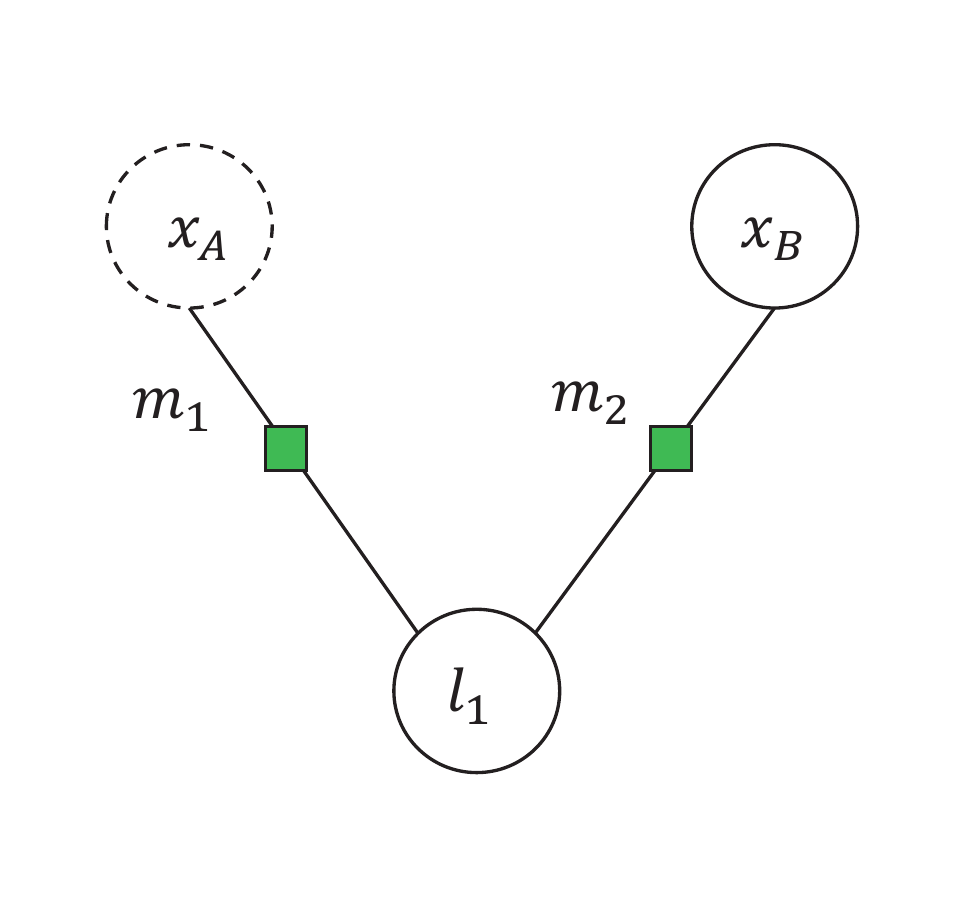}
	\caption{A local factor-graph illustrating feature/landmark triangulation.}
	\label{land}       
\end{figure}
\subsection{Landmark Degeneracy}
\label{sec32}
The landmark degeneracy cases (or called feature triangulation failure) are mainly caused by specific sonar motion \cite{huang2015towards}. Yang \textit{et al.} \cite{yang2017acoustic} had theoretically proved that in the 3D environment, any pure translational motion on $x, y$ directions or pure rotation around the $z-$axis, or any combination of the above motion would result in feature triangulation failure. Moreover, the features within other basic motions such as $x$ rotation, $y$ rotation, and $z$ translation required more sonar measurements to be triangulated.

Further, we give a complementary proposition and theoretical analysis based on these previous works, which will also be applied in this paper.

\begin{proposition}
	Consider a 2D imaging sonar with a narrow vertical aperture, the elevation angles of landmarks can be estimated using the planar assumption and method in \cite{negahdaripour20133}, which can eliminate the elevation ambiguity. Therefore, given the absolute (or reliable) onboard observations about z translation measured by depth meter and pitch/roll angles by the accurate inertial navigation system, the landmark can be triangulated during pure $x,~y$ motion or pure $z$ rotation (yaw), or the combinations of these motions, i.e., the triangulation failure/degeneracy in the three remaining motion can be eliminated.
	\label{prop1}
\end{proposition}

\begin{proof}
	Intuitively, we raise an elementary triangulation example of a feature/landmark. 
	As the local factor-graph shown in Fig.~\ref{land}, given two known poses $x_A$ and $x_B=[\psi_{x_B},t_{x_B}^x,t_{x_B}^y]^\mathrm T$, and the corresponding measurement factors $m_1$ and $m_2$,  the landmark $l_j$ is to be triangulated.
	{Thus, the sonar prediction function can be linearized to:}
	\begin{equation}
		h(l_j)=h(l_j^0+\Delta l_j)\approx h(l_j^0)+\mathbf{H}\Delta l_j,
	\end{equation}
	where $l_j^0$ is the linearization point taken as the feature coordinates from previous sonar observation. The $x_B$ is omitted in $h(\cdot)$ for brevity.
	It is worth noting that the unknowns in feature position $l_j =[\psi_j,~r_j]^\mathrm T$ to be estimated reduce to 2 and can be solved by the update vector $\Delta l_j$, and the related Jacobian matrix formed by one pose and one feature is $\mathbf{A}_l = \partial{h(l_j)}/{\partial{l_j}}\in \mathbb{R}^{2\times2}$. 

	1) \textit{Pure x motion}: In this case, considering a minor increment in $x_B$ to represent the pure $x$ motion in one time step yields:
	\begin{align}
		&\mathbf t_{x_B}=\left[\begin{array}{lll} \delta x  &0 \end{array}\right]^\mathrm T, \mathbf R_{x_B}=\mathbf I_{2\times2},\\ \nonumber
		&\mathbf{q}_j=\mathbf R_{x_B}(\mathbf{p}_j-\mathbf t_{x_B})=\left[\begin{array}{l} r_j \cos \psi_j-\delta x\\r_j \sin \psi_j \end{array}\right].
	\end{align}
	Therefore, based on Eq.~\eqref{eq:jacobian}, the current Jacobian is:
	\begin{align}
		&\mathbf{A}_l = \frac{\partial h(l_j)}{\partial l_j}=\begin{bmatrix} -\frac{q_{jy}}{q_{jx}^2+q_{jy}^2} &\frac{q_{jx}}{q_{jx}^2+q_{jy}^2}\\\frac{q_{jx}}{\|\mathbf q_j\|_2} &\frac{q_{jy}}{\|\mathbf q_j\|_2} \end{bmatrix}\frac{\partial{\mathbf q}_j}{\partial l_j}\\ \nonumber
		&=\begin{bmatrix} -\frac{r_j^2-r_j\delta x \cos \psi_j}{q_{jx}^2+q_{jy}^2} &\frac{-\delta x \sin \psi_j}{q_{jx}^2+q_{jy}^2} \\ \frac{r_j\delta x \sin \psi_j}{\sqrt{q_{jx}^2+q_{jy}^2}} &\frac{r_j-\delta x \cos \psi_j}{\sqrt{q_{jx}^2+q_{jy}^2}} \end{bmatrix},~rank(\mathbf{A}_l)=2.
	\end{align}
		
	Hence, the Jacobian matrix $\mathbf{A}_l$ and the corresponding Hessian matrix are both full-rank, which means that Eq.~\eqref{eq:normal} will be solvable uniquely and the landmark can be triangulated with an initial state $l_j^0=[\psi_j^0,~r_j^0]^\mathrm T$.
		
	2) \textit{Pure y motion}: This case is similar to the pure $x$ motion, so we give a brief result.
	\begin{align}
		&\mathbf t_{x_B}=\left[\begin{array}{ll} 0 \\ \delta y \end{array}\right], \mathbf R_{x_B}=\mathbf I_{2\times2}, \mathbf{q}_j=\left[\begin{array}{l} r_j \cos \psi_j \\r_j \sin \psi_j -\delta y \end{array}\right],\\ \nonumber
		&\frac{\partial h(l_j)}{\partial l_j}=\begin{bmatrix} -\frac{r_j^2-r_j\delta y \sin \psi_j}{q_{jx}^2+q_{jy}^2} &\frac{\delta y \cos \psi_j}{q_{jx}^2+q_{jy}^2} \\ \frac{-r_j\delta y \cos \psi_j}{\sqrt{q_{jx}^2+q_{jy}^2}} &\frac{r_j-\delta y \sin \psi_j}{\sqrt{q_{jx}^2+q_{jy}^2}} \end{bmatrix},\\ \nonumber
		&rank(\mathbf{A}_l)=2.
	\end{align}
		
	In this case, a full-rank Jacobian matrix $\mathbf{A}_l$ means the landmark can be triangulated uniquely too.
		
	3) \textit{Pure z rotation}: In this case, the translational vector $\mathbf t_{x_B}$ is zero, the rotation matrix and $\mathbf q_j$ are:
	\begin{equation}
		\mathbf R_{x_B}=\left[\begin{matrix} \cos \delta \psi & \sin \delta \psi \\-\sin \delta \psi & \cos \delta \psi \\  \end{matrix}\right], \mathbf q_j = \left[\begin{matrix} r_j \cos \psi_j\\r_j \sin \psi_j\end{matrix}\right].
	\end{equation}
	This yields:
	\begin{align}
		&\frac{\partial h(l_j)}{\partial l_j}=\left[\begin{matrix} -\frac{\sin \psi_j}{r_j} &\frac{\cos \psi_j}{r_j} \\ \cos \psi_j &\sin \psi_j \end{matrix}\right] 
		\left[\begin{matrix} \cos \delta \psi & \sin \delta \psi \\-\sin \delta \psi & \cos \delta \psi \\  \end{matrix}\right] \frac{\partial{\mathbf q}_j}{\partial l_j}\\ \nonumber
		&=\left[\begin{matrix} \cos \delta \psi & \frac{\sin \delta \psi}{r_j} \\-r_j \sin \delta \psi &\cos \delta \psi \end{matrix}\right].
	\end{align}
		
	In general, the incremental yaw angle $\delta \psi$ is quite small in a time step, i. e. $\sin \delta \psi \approx 0$ and $\cos \delta \psi \approx 1$, then we can get $rank(\mathbf{A}_l)=2$.
		
	4) \textit{Combination motion}: The landmark can also be triangulated within the combinations of the above primitive motion since the linear property. The detail is omitted here.
\end{proof}

\begin{remark}
	Note that the statement in \textit{Proposition 1} is equivalent to the reduced 2D case approximately.
\end{remark}
\begin{figure*}
	\centering
	\includegraphics[width=0.9\textwidth]{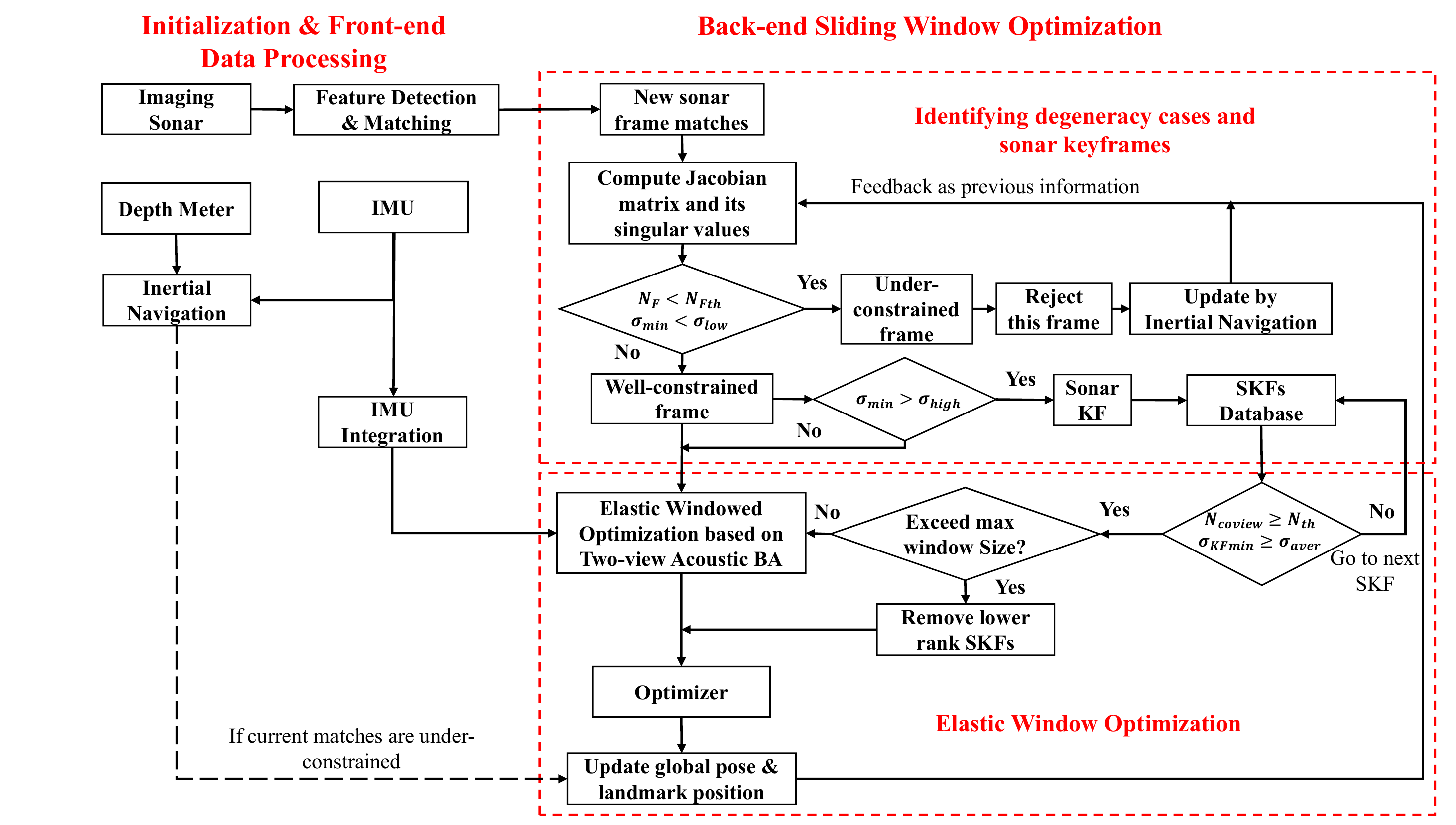}
	\caption{Block diagram of the proposed sonar KF-based elastic windowed optimization. As the main block, the back-end optimization includes under-constrained case detection, keyframes selection, and window size regulation.}
	\label{diag}       
\end{figure*}
%

For the degeneracy cases discussed above, we prefer to apply the singular-value decomposition (SVD) approach to solve this class of nonlinear state estimation problems since the good sensitivity of singular values to perturbations \cite{nocedal2006numerical,zhang2016degeneracy}.
Specifically, for an LLS problem, a minor singular value represents a less constrained update vector, i.e., a degeneracy case. This motivates us to utilize the singular values to discriminate the degeneracy cases and well-constrained cases in the following section.

\section{Sonar Keyframe-Based Sliding Window Optimization}
\label{sec4}
\subsection{Overview and Initialization}
\label{sec41}
Fig. \ref{diag} shows the proposed sonar keyframe-based elastic windowed optimization scheme.  
The initialization process mainly involves IMU calibration, IMU-Sonar extrinsic parameters estimation, and other sensors' initialization, which are omitted here. The front-end generates sonar features and feature matches between successive frames using A-KAZE and FLANN methods. The detected features in each frame are ordered with a unique id and saved to the feature database with pixel coordinates unless already in it or matched to existed features. The inertial navigation data between two consecutive sonar poses is completed by the method in Sec.~\ref{sec:ian}. 

\subsection{Identifying Degeneracy Cases and Sonar Keyframes}
\label{sec42}
As presented in \textit{Proposition 1}, given the reliable measurements on $z$, pitch, and roll motion, the degeneracy cases will only relate to pose ambiguity. 
A sonar frame would be regarded as an under-constrained one if: 
\begin{enumerate}
	\item the number of detected and matched features $N_{F}$ between two consecutive sonar frames is less than a threshold $N_{Fth}$; 
	\item the minimum singular value $\sigma_{min}$ of the Jacobian matrix $\mathbf{A}$ formed by poses and features in the two frames meets $\sigma_{min} < \sigma_{low}$.
\end{enumerate}

The under-constrained frames will be rejected outside of the sliding window optimization and the corresponding poses and landmarks will be updated only by inertial navigation. On the contrary, the current frame will be the SKF if $\sigma_{min}$ is larger than a threshold $\sigma_{high}$. The selected SKFs will be saved into the SKFs database as the candidate frames for the elastic window.
Thus we have already defined the saliency for sonar images to identify keyframe.

Note that the Jacobian matrix is computed in an analytical form and we get the singular values by SVD: $\mathbf{A}=\mathbf{USV^T}$, where $\mathbf{U}\in \mathbb R^{p\times p}$ and $\mathbf{V}\in \mathbb R^{q\times q}$ are orthogonal matrices, $\mathbf{S}=[\boldsymbol\sigma_q \ \mathbf 0_{p-q}]^{\rm T}\in \mathbb R^{p\times q}$ contains singular values $\sigma_1\geq \dots \geq\sigma_q$. 
{Generally, for a degeneracy case, the minimum singular value $\sigma_{min}$ of the Jacobian matrix is, inevitably, a positive number and very close to zero, which represents a under-constrained direction in the optimization, according to \cite{zhang2016degeneracy}. Therefore, $\sigma_{low}$ is set as a small positive empirically and the $\sigma_{high}$ is set to 5-8 times of the $\sigma_{low}$ to choose the SKFs.}
$N_{Fth}$ derives from the existence of solutions for the LLS problem mentioned in Sec. \ref{sec31}, that is, we need $M\geq N_{Fth}= 6$ in 3D scene and $M\geq N_{Fth}= 2$ in 2D scene to ensure the LLS equations to be well-constrained, or so-called non-underdetermined. 

\subsection{Adjusting the Elastic Window Size and Optimization}
\label{sec43}
Here we traverse the SKF database and add the qualified SKFs into the sliding window based on the following rules:
\begin{enumerate}
	\item the number of the same feature ids between the candidate KF and current frame needs to be larger than a threshold $N_{coview}\geq N_{th}$;
	
	\item On this basis, we rank the SKFs based on the magnitude of minimum singular value in descending order, the KF with a greater $\sigma_{min}$ than the average minimum singular value will be added into the window, i.e.:
	\begin{align}
		\sigma_{KFmin}\geq \sigma_{aver}=\frac{1}{N_S}\sum_s^{N_S}{\sigma_{min}^s}.
	\end{align}
\end{enumerate}

Considering maintain an appropriate optimization scale, the window size $N_S$ needs to be adjusted automatically in a range of $[2,N_{Smax}]$ based on two-view ABA. 
If the window size exceeds the boundary $N_{Smax}$, we only reserve the SKFs which rank top $N_{Smax}-2$, other SKFs are removed from the sliding window.
$N_{Smax}$ and $N_{th}$ are chosen, depending on the specific mission.
{Note that our method needs no additional landmarks even if the indispensable number of landmarks $N_{Fth}$ cannot be satisfied for a while, in which case we can use inertial navigation and past SKFs to improve the current accuracy.}

Once the sonar frames in the sliding window are determined, the NLS problem will be well-conditioned approximately, and can also be solved by, e.g., Gauss-Newton or Levenberg-Marquardt method in GTSAM \cite{GTSAM2017}.

\section{Results and Discussions}
\label{sec5}
To validate the proposed method, we first conduct two underwater simulations on the high-fidelity physical model-based simulation platform named Gazebo UUV Simulator \cite{manhaes2016uuv}, and then an evaluation on a public dataset. We compare our method with the two-view ABA, which was proposed in \cite{ShinBundle} and used in \cite{westman2019degeneracy}. All simulations and experiments are conducted on a desktop PC with a 3.6~GHz Intel i3-9100F CPU and 16G RAM.

In the first two simulations, as in Fig.~\ref{rov}, we use an ROV equipped with imaging sonar, IMU, depth meter, and other onboard sensors to follow a pre-defined 10~m $\times$ 10~m squared path. Note that this path is the combination of pure $x,~y$ motion, and pure $z$ rotation, the triangulation degeneracy will not exist given accurate depth, roll, and pitch measurements in this combined motion according to \textit{Proposition 1}.
The ROV keeps a relatively low speed and senses the fixed artificial landmarks along the path. Important parameters are listed in Table \ref{tab:param}. 
Note that since we aim to verify the open-loop performance,  the loop closure mechanism is not used in the following experiments, though our proposed method is compatible with it.
\begin{figure}
	\centering
	\includegraphics[width=0.38\textwidth]{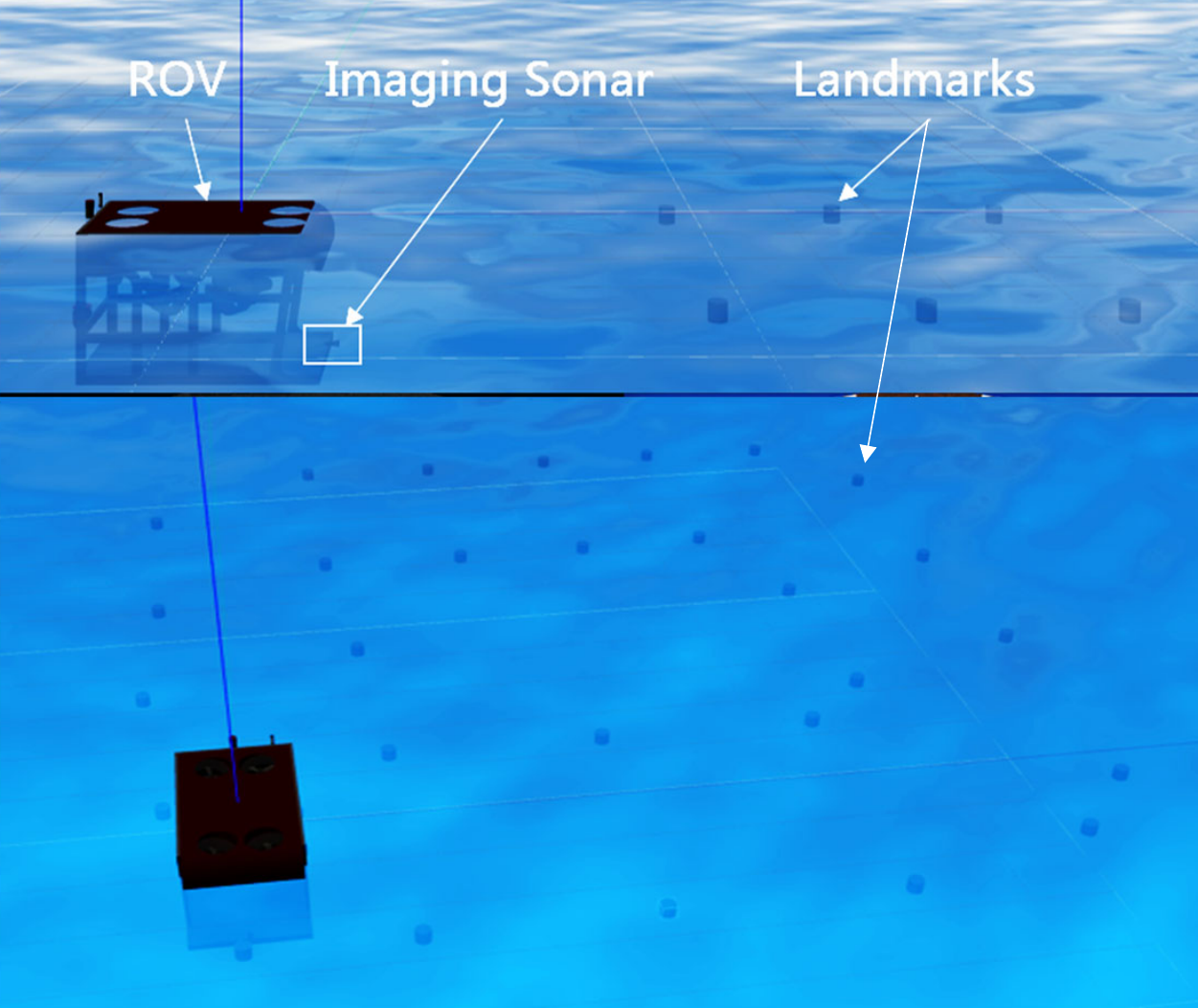}
	\caption{An ROV mounted with a forward-looking imaging sonar is collecting data along a squared path in the Gazebo UUV Simulator.}
	\label{rov}       
\end{figure}
\begin{table}[htbp]
    \centering
	\caption{Important Parameters.}
		\begin{tabular}{ll}
			\hline
			Items                                                         & Values           \\ \hline
			Sonar range {[}$r_{min}$,$r_{max}${]} (m)                     & [0.5,9]      \\
			Sonar bearing FOV {[}$\psi_{min}$,$\psi_{max}${]} (rad)       & [-$\pi$/4, $\pi$/4]                 \\
			Sonar elevation FOV {[}$\theta_{min}$,$\theta_{max}${]} (rad) & [-$\pi$/18, $\pi$/18]                 \\
			Sonar range noise $\sigma$ (m)                                & 0.05             \\
			Sonar bearing noise $\sigma$ (rad)                            & 0.02             \\
			Odometry translational noise $\sigma$ (m)                     & 0.05             \\
			Odometry rotational noise $\sigma$ (rad)                      & 0.02             \\
			Sonar position w.r.t ROV (m)                                  & {[}1.4,0,-0.6{]} \\
			Sonar azimuth w.r.t ROV (rad)                                 & {[}0,0,0{]}      \\ \hline
		\end{tabular}
	\label{tab:param}
\end{table}
Since we only concern about the horizontal pose and landmark estimation, we set the imaging sonar at the same depth as all landmarks, so these small elevation angles can be approximated to 0, in other words, this scene is similar to the planar 2D case. Therefore, the minimum number of tracked features is chosen as $N_{Fth}=2$. We also choose the values of $\sigma_{low}=0.13$ and $\sigma_{high}=0.8$ based on the analysis in Section V.~B and experimental experience. Besides, the maximum window size $N_{Smax}$ and the minimum number of co-viewed features $N_{th}$ is set to 5 and 4, respectively. Note that this 2D simplification would introduce additional feature position errors according to Eq.~(1).
As in Fig.~\ref{kaze}, the features in the consecutive simulated sonar images are detected and matched by A-KAZE \cite{alcantarilla2011fast} and FLANN \cite{muja2009fast} approaches in the front-end. We also use the ratio test (ratio among 0.4-0.6) to reject the wrong matches in the front-end coarsely and post-process the remaining ones using our scheme since outliers may still exist. We choose experimentally a ratio of 0.48 in the simulations.

\begin{figure}
	\centering
	\includegraphics[width=0.45\textwidth]{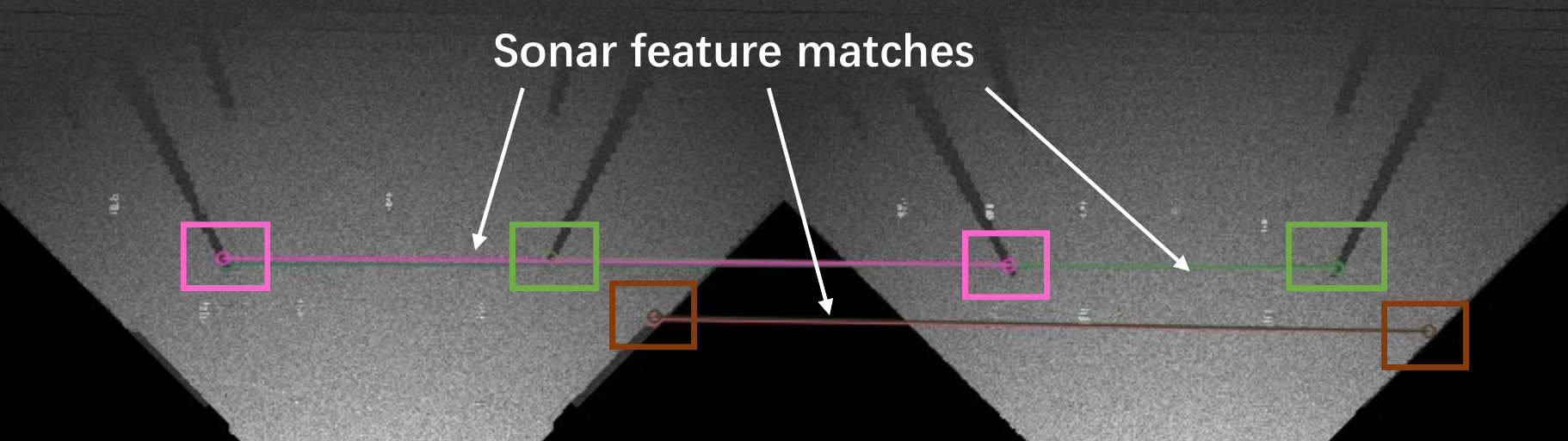}
	\caption{Acoustic feature matching in the ROI of simulated sonar images using the A-KAZE algorithm {after the ratio test}. This method performs better in the balance of time cost and accuracy rate than other methods such as ORB, SIFT, KAZE, and SURF \cite{westman2018feature}.}
	\label{kaze}       
\end{figure}
%
\begin{figure}
	\centering
	\includegraphics[width=0.4\textwidth]{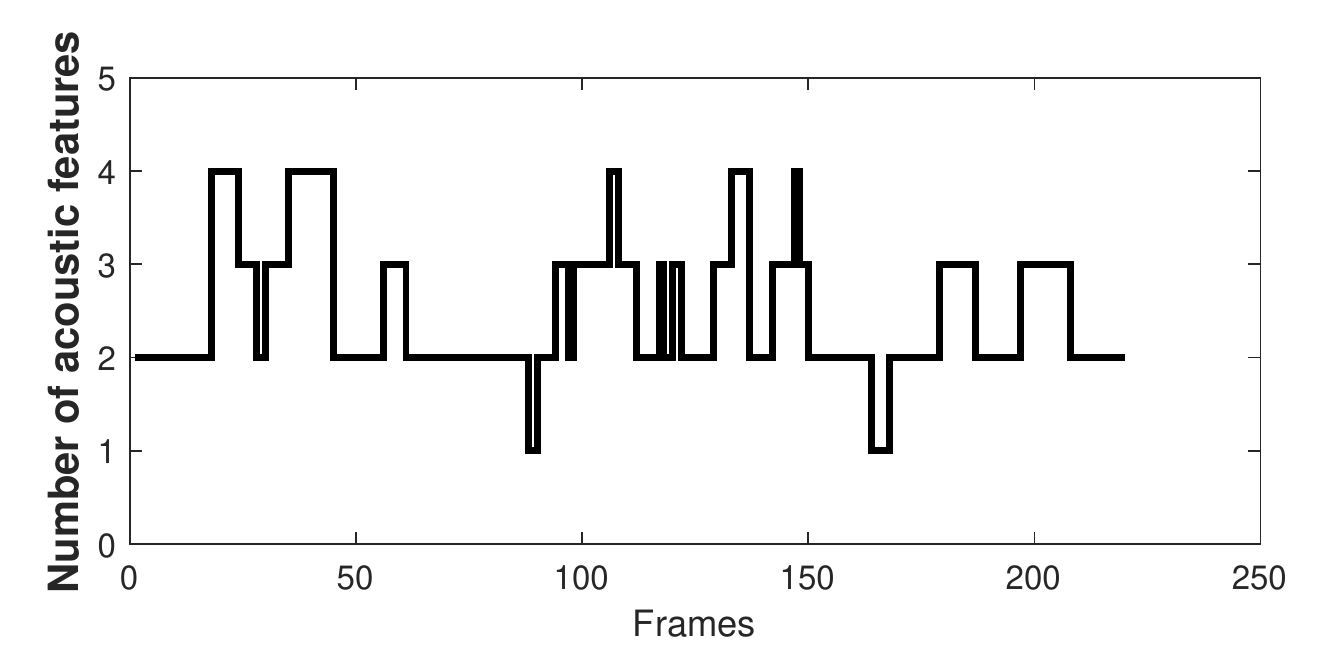}
	\caption{The number of available acoustic features at each frame after the ratio test. The number of features reduces to 1 at frames 88 - 89 and 164 - 167, less than the minimum requirement $N_{Fth}=2$, which implies the under-constrained frames appear.}
	\label{num_fea}       
\end{figure}
\begin{figure}
	\centering
	\subfigure[The top-down view of trajectories and landmarks]{
		\begin{minipage}[t]{0.8\linewidth}
			\centering
			\includegraphics[width=1.0\textwidth]{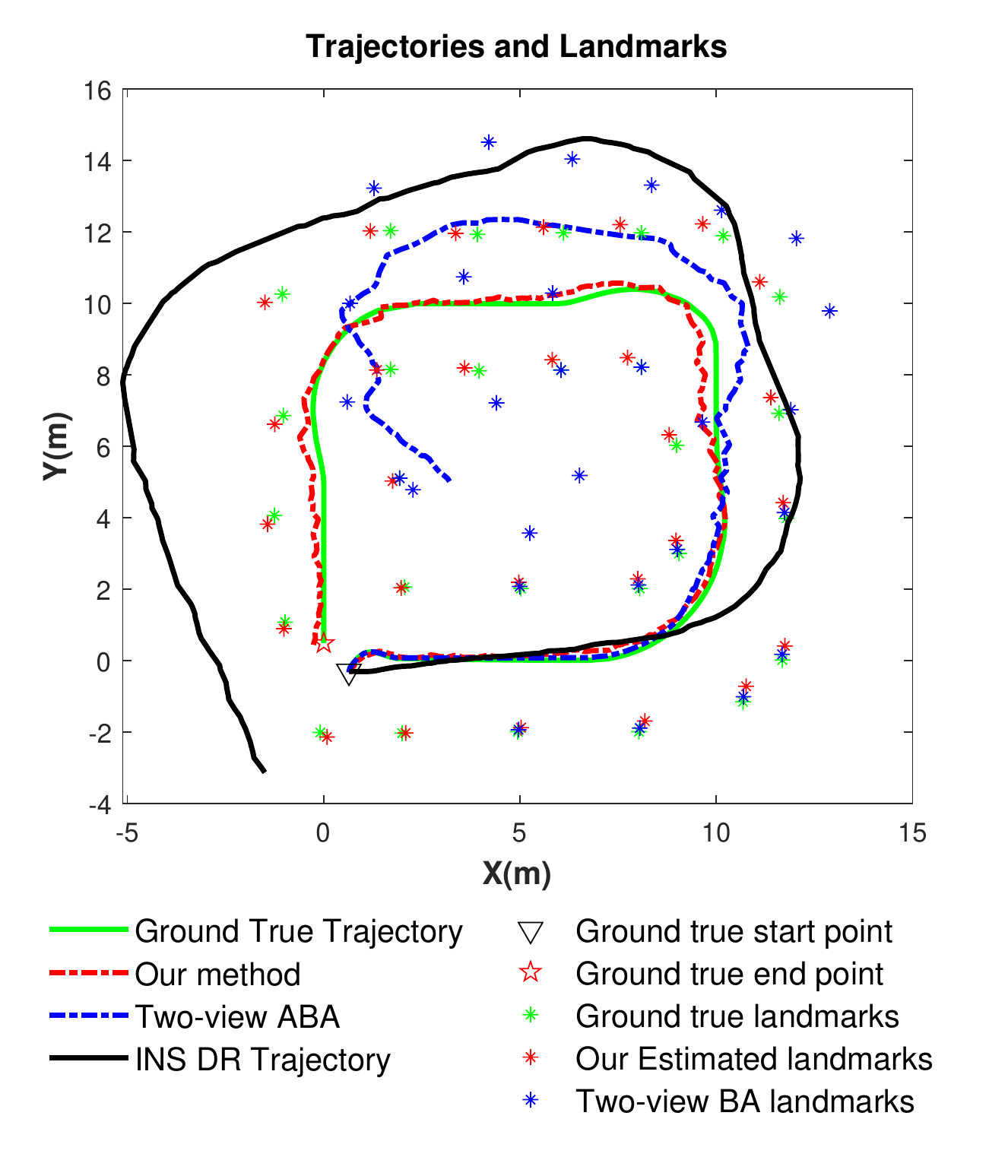}
		\end{minipage}%
	}%
	\quad
	\subfigure[Average trajectory position RMSE]{
		\begin{minipage}[t]{0.8\linewidth}
			\centering
			\includegraphics[width=1.0\textwidth]{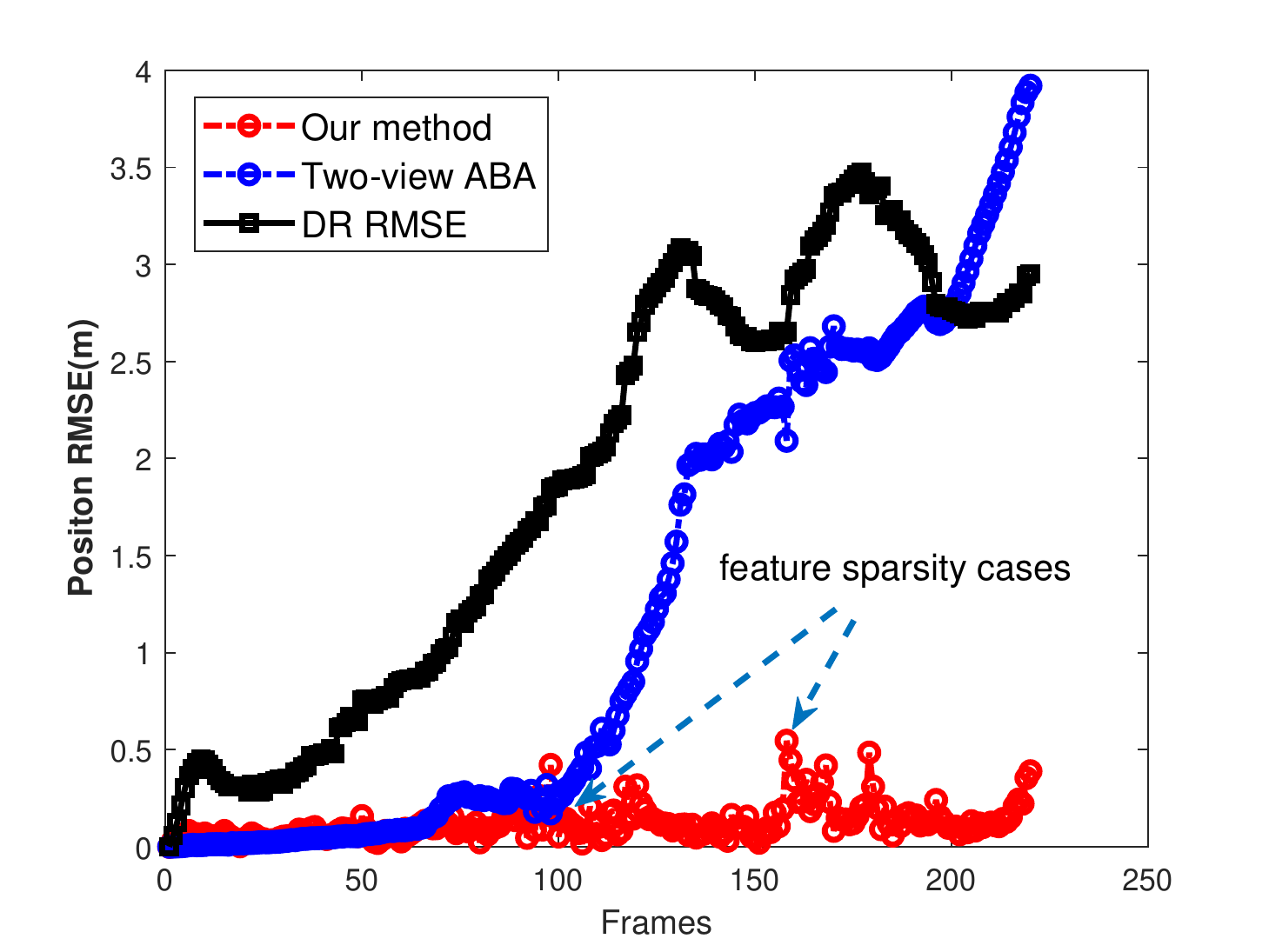}
		\end{minipage}%
	}%
	\quad
	\subfigure[Average landmark position RMSE and MAE]{
		\begin{minipage}[t]{0.7\linewidth}
			\centering
			\includegraphics[width=1.0\textwidth]{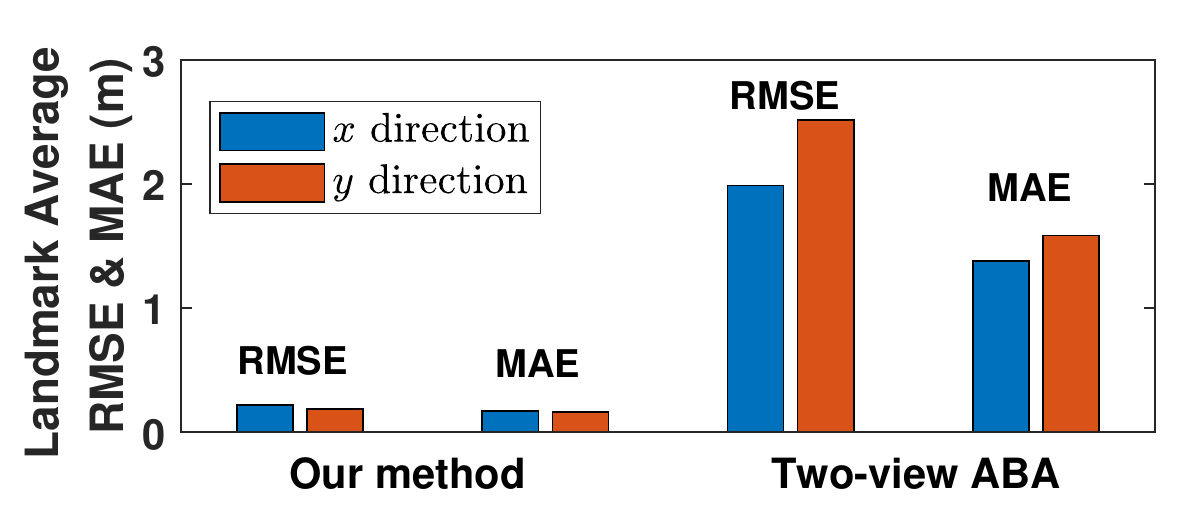}
		\end{minipage}%
	}%
	\caption{Simulation results of sparse features case.}
	\label{normal}       
\end{figure}
\subsection{Simulation A: Sparse Features}
\label{sec51}
To simulate the feature sparsity in the underwater environment, we place 28 small landmarks along the path and only 25 landmarks are measured. The number of detected acoustic features at each sonar frame is shown in Fig.~\ref{num_fea}. It's worth noticing that the number of features reduces to 1 at frames 88 - 89 and 164 - 167, less than the minimum requirement $N_{Fth}=2$ as analyzed in Sec.~\ref{sec32}, which would potentially cause undesired results. 

Fig.~\ref{normal}(a) compares the trajectories and landmarks estimated by our method and two-view ABA with the ground truth. The DR trajectory is also illustrated as a reference. Fig.~\ref{normal}(b) shows the corresponding RMSEs of different methods. Fig.~\ref{normal}(c) gives the RMSEs and mean absolute errors (MAEs) of all estimated landmarks at $x,~y$ directions. 

In the first one-third of the path, the proposed method outperforms the two-view ABA (both much better than the INS DR) since the well-constrained sonar keyframes are added into the elastic window. However, after encountering the first feature sparsity case at about frame 88, the ABA begins to accumulate larger errors both in pose and landmark estimation, then gets even worse and collapses when meeting the second sparsity case. According to the analysis in Sec.~\ref{sec3}, this phenomenon mainly attributes to the insufficient features and the resulting under-determined least square equations. 
In contrast, the RMSE curve of our method fluctuates when these two sparsity cases occur, but the error is always bounded within about 0.5 m even at the consecutive under-constrained frames 164 - 167 because of the presence of degeneracy identification.
Moreover, the continuous fluctuation in a period of RMSE after the sparsity cases occur exactly indicates the adverse effects of under-constrained frames on feature-based sonar localization and navigation.

Notably, our method limits the effects in a small range and shows the effectiveness on degeneracy cases.
The average position RMSE and MAE of detected landmarks also show the estimation accuracy of the proposed method.
\subsection{Simulation B: Add Wrong Associations}
\label{sec52}
To further verify the robustness of the proposed method, we add two wrong associations at frame 50 to simulate the wrong associations that occurred inevitably in the front-end. Specifically, we manually associate landmarks 6 and 7 to 1 and 8, respectively. 

In this case, as shown in Fig.~\ref{outlier}(a), the estimated trajectory of two-view ABA starts to deviate the path seriously after adding the wrong feature matches, the large deviation does not get corrected in the remaining path for a long time. The sudden increase and quick accumulation of trajectory RMSE in Fig.~\ref{outlier}(b) also evident this. 

On the contrary, the dashed rectangle in Fig.~\ref{outlier}(a) shows that the deviation in the trajectory of our method is corrected almost instantaneously and the pose estimation always keeps near consistent with the ground truth later. This mainly owes to the minimum singular value $\sigma_{low}=0.13$ mentioned earlier identifying the degeneracy cases. The RMSE curve of our method is shown in Fig.~\ref{outlier}(b). The overall average landmark position RMSE and MAE in Fig.~\ref{outlier}(c) are larger than the ones in Fig.~\ref{normal}(c), but the amplitudes are still limited to 0.5 m with our method, while the two-view ABA results in errors larger than 1 m in both experiments.

Hence, this experiment illustrates our method is robust and resilient to the outliers from the front-end.
As evident from these experimental results, the proposed method achieves the expected performance.

\begin{figure}
	\centering
	\subfigure[The top-down view of trajectories and landmarks]{
		\begin{minipage}[t]{0.8\linewidth}
			\centering
			\includegraphics[width=1\textwidth]{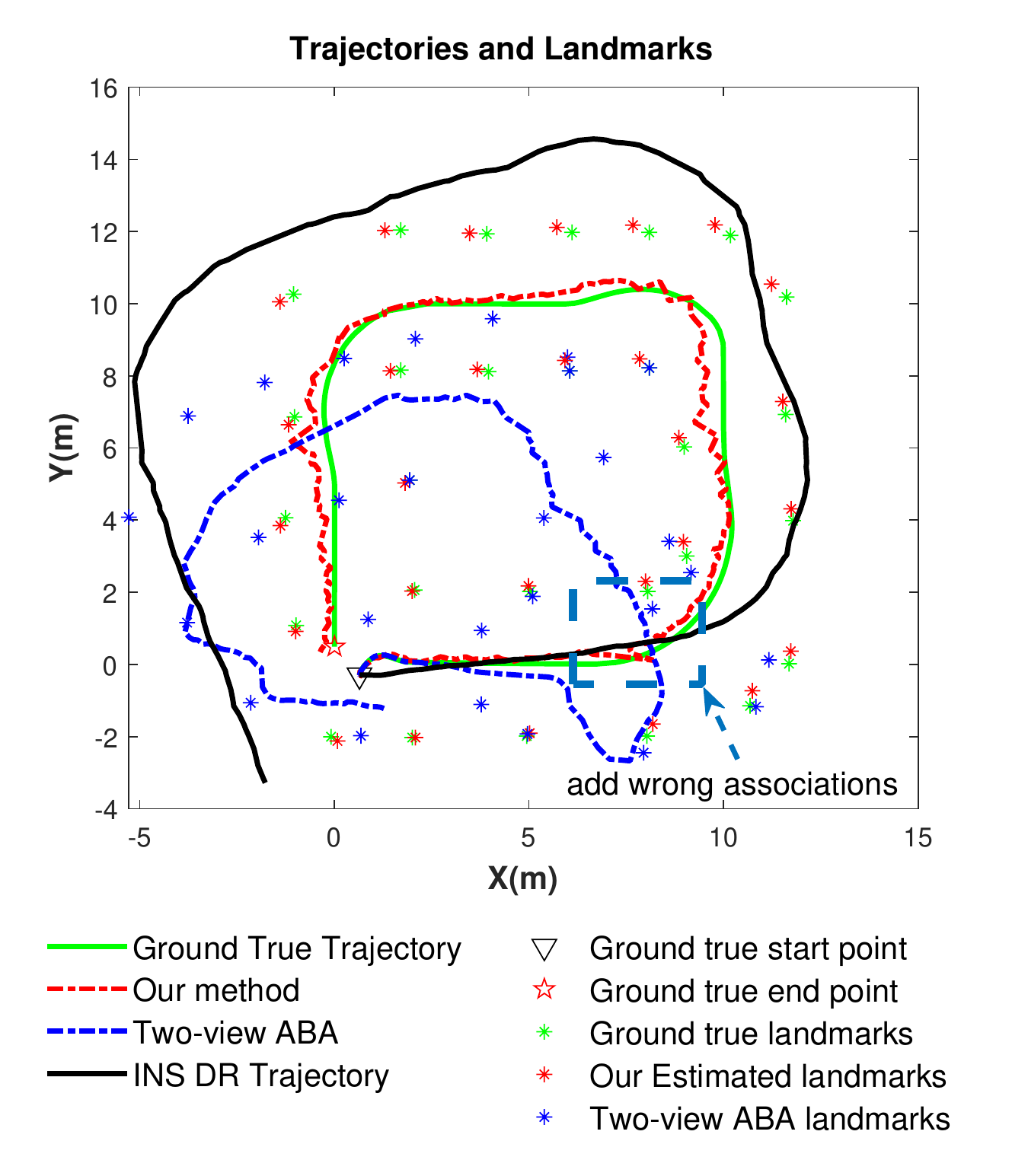}
		\end{minipage}%
	}%
	\quad
	\subfigure[Average trajectory position RMSE]{
		\begin{minipage}[t]{0.8\linewidth}
			\centering
			\includegraphics[width=1\textwidth]{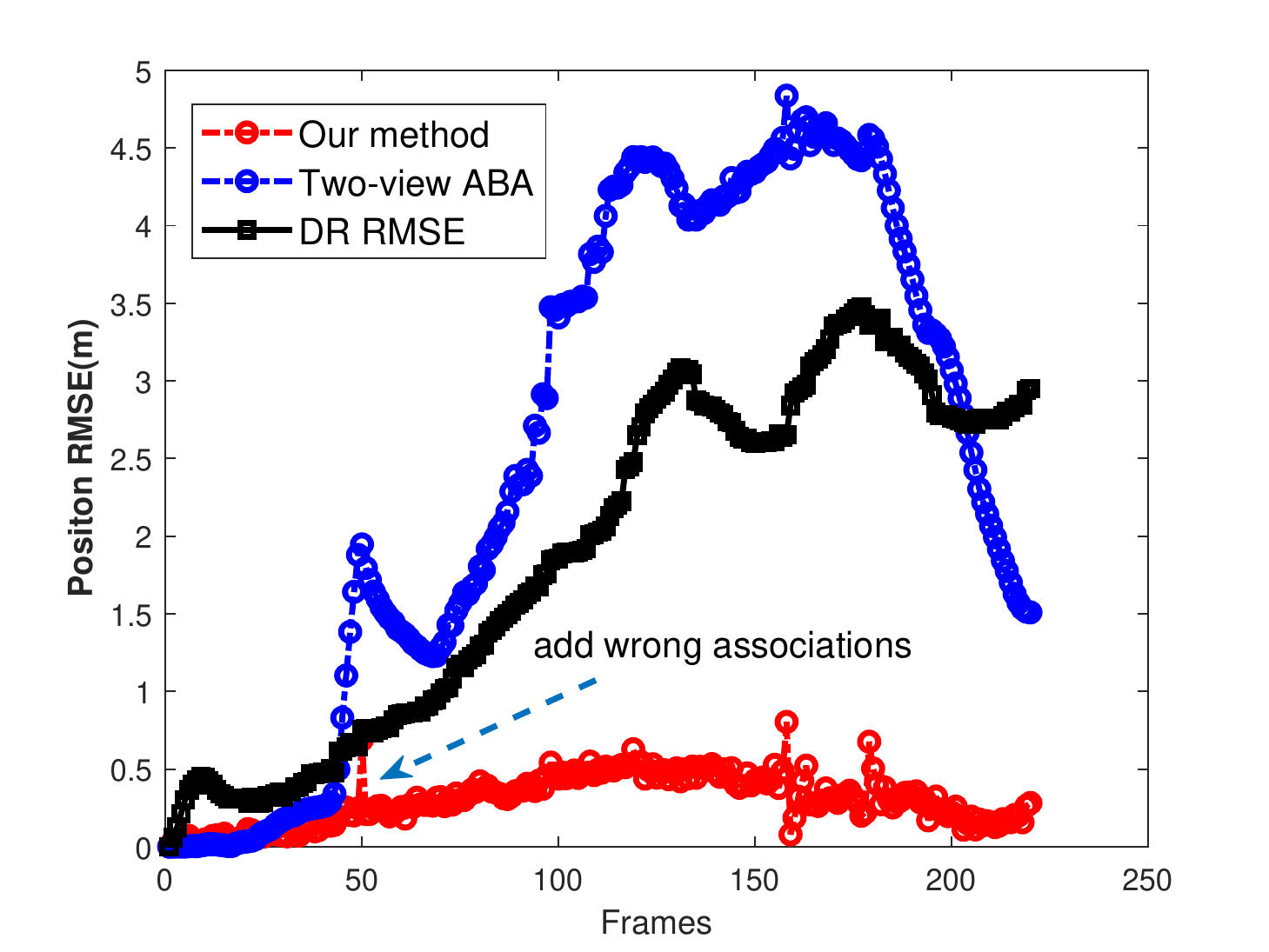}
		\end{minipage}%
	}%
	\quad
	\subfigure[Average landmark position RMSE and MAE]{
		\begin{minipage}[t]{0.7\linewidth}
			\centering
			\includegraphics[width=1\textwidth]{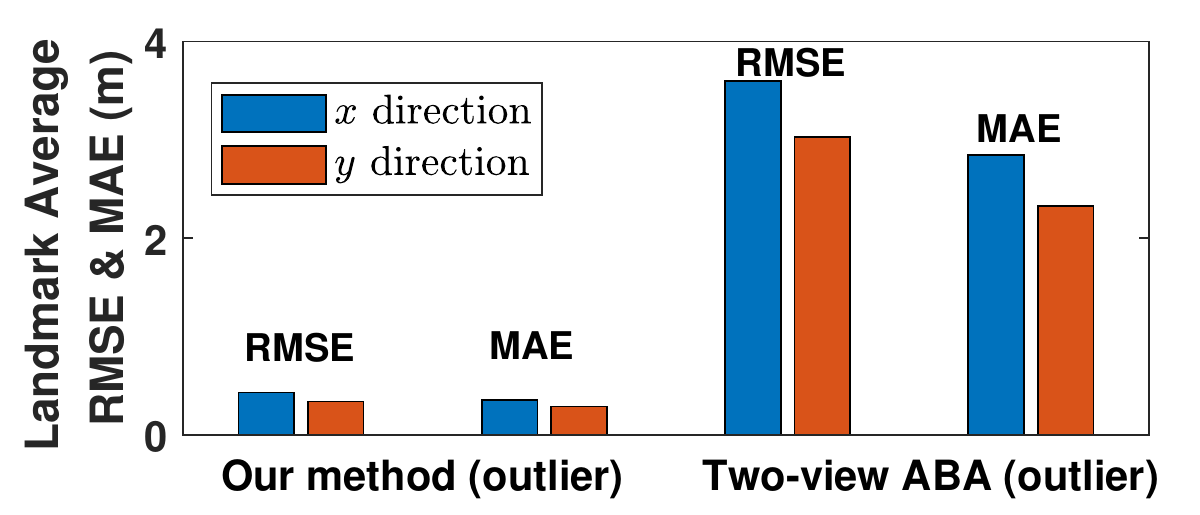}
		\end{minipage}%
	}%
	\caption{Simulation results of wrong association case.}
	\label{outlier}       
\end{figure}
\subsection{Evaluation on a Public Dataset}
To validate the practical benefit of our proposed method, we adopt the dataset collected using the Ictineu AUV by Ribas \cite{ribas2008underwater} in an abandoned marina (see Fig.~\ref{marina}(a)). {The vehicle was mounted with a Tritech Miniking imaging sonar and onboard navigation sensors such as a depth sensor and an Xsens MTi motion reference unit (MRU). The AUV traveled along a 600 m path with a relatively low speed (about 0.2 m/s) using more than 45 minutes. Note that this case can be also treated as a 2D case because of the above onboard sensors, so the collected path is approximated as the combination of pure $x,~y$ motion, and pure $z$ rotation, which accords with the statement of \textit{Proposition 1}.} The omnidirectional mechanically scanned imaging sonar was set to a maximum range of 50 m, a resolution of 0.1 m, and a step angle of 1.8$^\circ$. The data collected by a differential global positioning system (GPS) on a buoy was used as the ground truth.

Fig.~\ref{marina}(b) shows the resulting trajectories in this experiment, and the comparative position RMSEs of estimated trajectories are in Fig.~\ref{marina}(c). 
{The red markers ``+'' depict the sonar features extracted from sonar images along the path and the blue markers ``+'' represent the salient features can be chosen from the sonar features by a voting algorithm \cite{ribas2008underwater}, which compute the parameters of all candidate features through a voting model to get a complete voting space, then find the salient features with higher votes than a threshold. For more details please see \cite{ribas2008underwater}.}
Note that we use no loop closure detection in this experiment and we prefer to use position RMSE to measure the real performance since the features have no ground truth as reference {in the evaluation.}

{The error analyses of our method, two-view ABA and INS DR are list in Table~II. The absolute errors between the start and end point position of the above three methods are 2.247 m, 4.914 m and 32.818 m, respectively.}
{Similar to the former results, our proposed method shows the effectiveness and accuracy in field applications and performs better than others in pose estimation.}
\begin{figure}
	\centering
	\subfigure[Left: The Ictineu ROV \cite{ribas2008underwater}.~~
	Right: A satellite image of the abandoned marina from Google Earth.]{
		\begin{minipage}[t]{0.8\linewidth}
			\centering
			\includegraphics[width=1\textwidth]{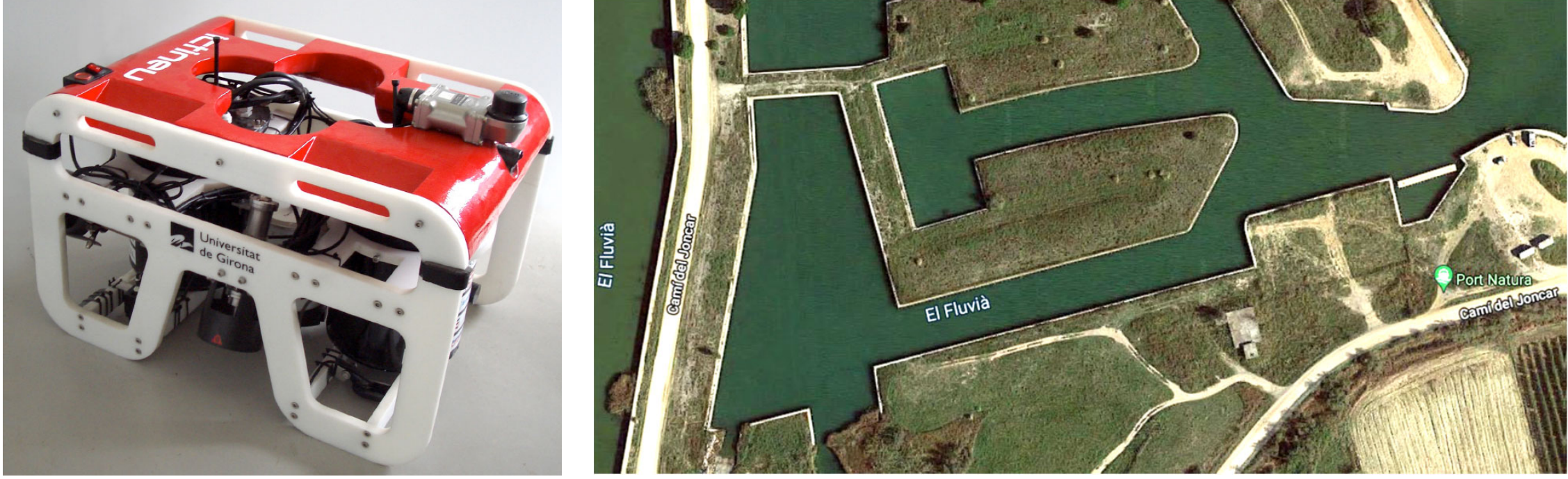}
		\end{minipage}%
	}%
	\quad
	\subfigure[The top-down view of estimated trajectories]{
		\begin{minipage}[t]{0.8\linewidth}
			\centering
			\includegraphics[width=1\textwidth]{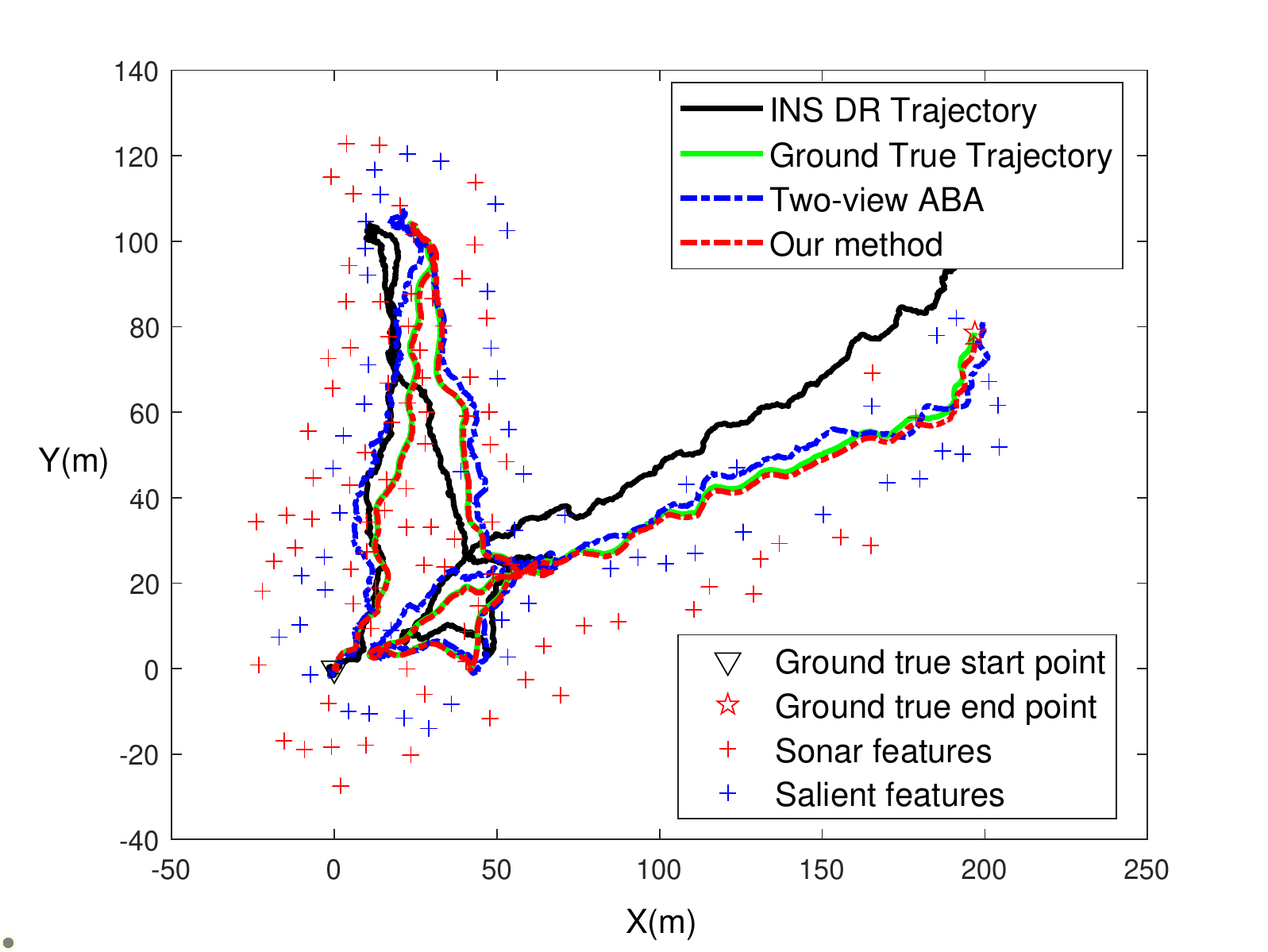}
		\end{minipage}%
	}%
	\quad
	\subfigure[Average trajectory position RMSE in marina dataset]{
		\begin{minipage}[t]{0.8\linewidth}
			\centering
			\includegraphics[width=1\textwidth]{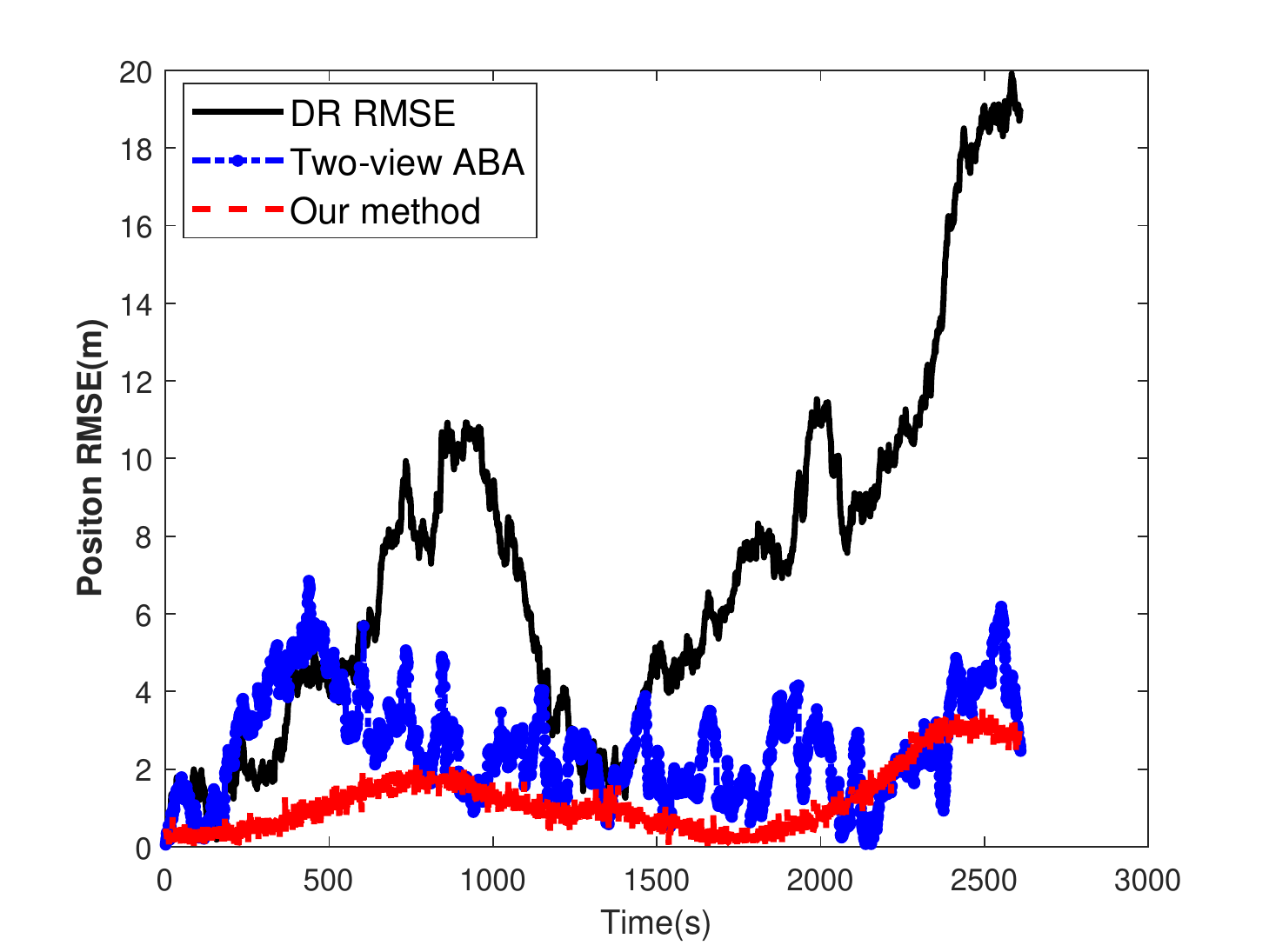}
		\end{minipage}%
	}%
	\caption{Experimental results of marina dataset at St. Pere, Spain.}
	\label{marina}       
\end{figure}
\begin{table}[htbp]
\centering
\caption{Error analysis of position RMSE in St. Pere dataset.}
    \begin{tabular}{|l|l|l|l|}
    \hline
    (Unit: meter)      & Our method               & Two-view ABA       & ~~~INS DR  \\ \hline
    mean $\pm$ std dev & 1.19 $\pm$ 0.84      & 2.57 $\pm$ 1.33        & 7.14 $\pm$ 4.66    \\ \hline
    {[}min, max{]} RMSE  & {[}0.02, 3.58{]} & {[}0.05, 6.85{]} & {[}0.11, 19.91{]} \\
    \hline
    \end{tabular}
\label{st-pere}
\end{table}
\begin{figure}
	\centering
	\includegraphics[width=0.4\textwidth]{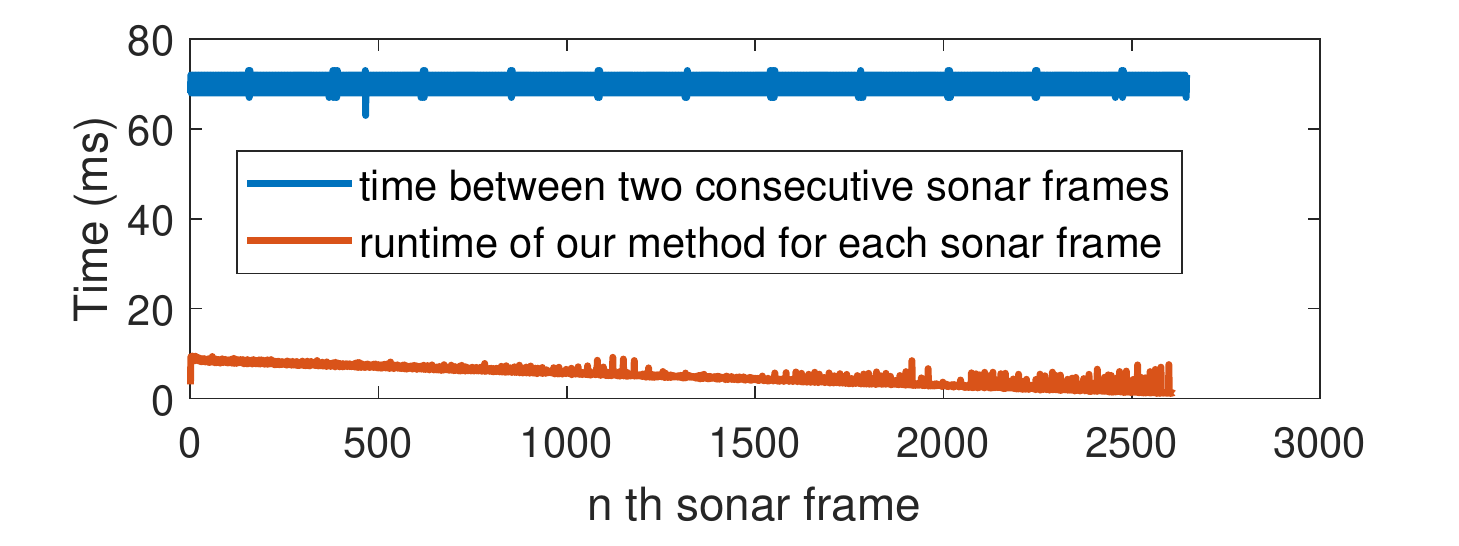}
	\caption{Runtime analysis of our method in the evaluation on the open dataset.}
	\label{runtime}       
\end{figure}
\subsection{Time Efficiency Analysis}
Moreover, as Fig.~\ref{runtime} shows, the execution time of our method {is lower than 10 ms (4.96 ms averagely), which is} far less than the time between two consecutive sonar scans (about 68-72 ms), this means that our method can process the current sonar frame earlier before a new frame comes. This also demonstrates the real-time performance of our proposed method intuitively.

More specifically, the front-end processing spends a major computational cost (about 3.7 ms averagely) of the proposed method, especially the feature extraction and matching, while the back-end cost less time since the sliding window optimization scale is relatively small.  
Additionally, the elastic window optimization module in the back-end spends an average time cost of 1.12 ms (about 89\% of the back-end cost).
Therefore, as the red line in Fig.~\ref{runtime}, the runtime diminishes as the salient features get sparse, {and the runtime fluctuates frequently at the later stage because the elastic window begins to add more SKFs to optimize, especially at the tail. However, the overall trend is still decreasing and no more than 10 ms,} showing the ability for online applications of the proposed method.
\section{Conclusions and Future Work}
\label{sec6}
This paper mainly contributes a novel sonar keyframe-based underwater autonomous localization and navigation method for low-cost AUVs. 
We raise the concept of sonar keyframes and design the selection criteria for discriminating well-constrained and under-constrained sonar frames. We also design an elastic sliding window optimization framework and add the well-constrained frames by certain rules. This framework allows us to fuse the inertial measurements and improve the potentially degenerate back-end optimization caused by under-constrained frames.
Our proposed method shows notable pose/landmark estimation improvement than two-view ABA and INS DR in simulations, as well as a 1.38 m and 5.95 m decreasing of RMSE than these two methods in a $110~m\times200~m$ trajectory estimation of the evaluation on the marina dataset. This method also shows real-time performance and robustness to outliers, and needs no additional landmarks even feature sparsity occurs.
Our future work mainly involves underwater mapping and conducting more field experiments.

\ifCLASSOPTIONcaptionsoff
  \newpage
\fi



\bibliographystyle{IEEEtran}
\bibliography{IEEEabrv,tim} 
\end{document}